%% file: paper.tex
\documentclass{article}

\usepackage[final]{neurips_2022}

\input{math_commands.tex}

\usepackage[dvipsnames]{xcolor}
\definecolor{linkcolor}{RGB}{74, 102, 146}
\usepackage[colorlinks=true,allcolors=linkcolor,pageanchor=true,plainpages=false,pdfpagelabels,bookmarks,bookmarksnumbered]{hyperref}

\definecolor{mygray}{gray}{0.95}

\usepackage[utf8]{inputenc} 
\usepackage[T1]{fontenc}    
\usepackage{url}            
\usepackage{booktabs}       
\usepackage{amsfonts}       
\usepackage{amssymb}
\usepackage{amsmath}
\usepackage{amsthm}
\usepackage{nicefrac}       
\usepackage{microtype}      
\usepackage{lipsum}		
\usepackage{graphicx}
\usepackage{natbib}
\usepackage{doi}
\usepackage[short]{optidef}
\usepackage{cleveref}
\usepackage{thmtools,thm-restate}
\usepackage{enumitem}
\usepackage{wrapfig}
\usepackage{subcaption}
\usepackage[normalem]{ulem}
\usepackage{tikz}
\usetikzlibrary{arrows.meta}

\usepackage{xspace}
\newcommand*{\eg}{{\it e.g.}\@\xspace}
\newcommand*{\ie}{{\it i.e.}\@\xspace}

\newcommand{\norm}[1]{\left\lVert#1\right\rVert}
\newcommand*{\tran}{^{\mkern-1.5mu\mathsf{T}}}

\renewcommand{\div}{\text{div}}

\title{Neural Conservation Laws: \\A Divergence-Free Perspective}

\author{%
  Jack Richter-Powell \\
  Vector Institute \\
  \texttt{jack.richter-powell@mcgill.ca} \\
  \And
  Yaron Lipman \\
  Meta AI \\
  \texttt{ylipman@meta.com} \\
  \And
  Ricky T. Q. Chen \\
  Meta AI \\
  \texttt{rtqichen@meta.com} \\
}

\begin{document}

\maketitle

\begin{abstract}
We investigate the parameterization of deep neural networks that by design satisfy the continuity equation, a fundamental conservation law.
This is enabled by the observation that any solution of the continuity equation can be represented as a divergence-free vector field.
We hence propose building divergence-free neural networks through the concept of differential forms, and with the aid of automatic differentiation, realize two practical constructions.
As a result, we can parameterize pairs of densities and vector fields that always exactly satisfy the continuity equation, foregoing the need for extra penalty methods or expensive numerical simulation. 
Furthermore, we prove these models are universal and so can be used to represent any divergence-free vector field.
Finally, we experimentally validate our approaches by computing neural network-based solutions to fluid equations, solving for the Hodge decomposition, and learning dynamical optimal transport maps.
\end{abstract}

\section{Introduction}

Modern successes in deep learning are often attributed to the expressiveness of black-box neural networks. These models are known to be universal function approximators~\citep{hornik1989multilayer}---but this flexibility comes at a cost. 
In contrast to other parametric function approximators such as Finite Elements \citep{Schroeder_2017}, it is hard to \textit{bake} exact constraints into neural networks. 
Existing approaches often resort to penalty methods to approximately satisfy constraints---but these increase the cost of training and can produce inaccuracies in downstream applications when the constraints are not exactly satisfied. 
For the same reason, theoretical analysis of soft-constrained models also becomes more difficult.
On the other hand, enforcing hard constraints on the architecture can be challenging, and even once enforced, it is often unclear whether the model remains sufficiently expressive within the constrained function class.

In this work, we discuss an approach to directly bake in two constraints into deep neural networks: (i) \textbf{having a divergence of zero}, and (ii) \textbf{exactly satisfying the continuity equation}. 
One of our key insights is that \textit{the former directly leads into the latter}, so the first portion of the paper focuses on divergence-free vector fields.
These represent a special class of vector fields which have widespread use in the physical sciences. 
In computational fluid dynamics, divergence-free vector fields are used to model incompressible fluid interactions formalized by the Euler or Navier-Stokes equations.
In $\R^3$ we know that the curl of a vector field has a divergence of zero, which has seen many uses in graphics simulations (\eg, \citet{eisenberger_divergence-free_2018}). 
Perhaps less well-known is that lurking behind this fact is the generalization of divergence and curl through differential forms \citep{cartan1899some}, and the powerful identity $d^2=0$.
We first explore this generalization, then discuss two constructions derived from it for parmaterizing divergence-free vector fields. While this approach has been partially discussed previously~\citep{barbarosie2011representation,kelliher2021stream}, it is not extensively known and to the best of our knowledge has not been explored by the machine learning community. 

Concretely, we derive two approaches---visualized in \Cref{fig:fig1}---to transform sufficiently smooth neural networks into divergence-free neural networks, made efficient by automatic differentiation and recent advances in vectorized and composable transformations~\citep{jax2018github,functorch2021}.
Furthermore, both approaches can theoretically represent any divergence-free vector field.

We combine these new modeling tools with the observation that solutions of the continuity equation---a partial differential equation describing the evolution of a \textit{density} under a \textit{flow} ---can be characterized jointly as a divergence-free vector field. As a result, we can parameterize neural networks that, by design, always satisfy the continuity equation, which we coin Neural Conservation Laws (NCL). While prior works either resorted to penalizing errors \citep{raissi_physics-informed_2019} or numerically simulating the density given a flow \citep{chen2018neural}, baking this constraint directly into the model allows us to forego extra penalty terms and expensive numerical simulations. 

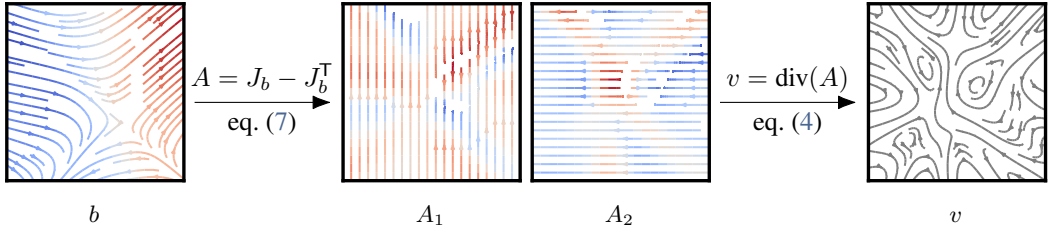
\begin{figure}
    \centering
    \begin{subfigure}[b]{0.18\linewidth}
        \includegraphics[width=\linewidth]{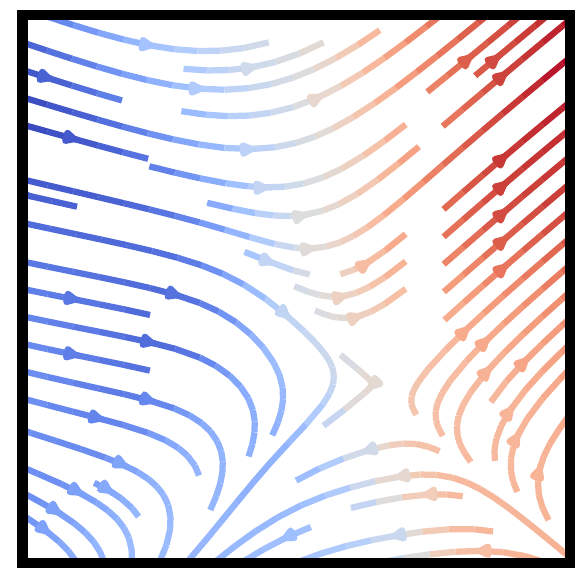}
        \caption*{$b$}
    \end{subfigure}%
    \begin{subfigure}[b]{0.14\linewidth}
        \centering
        \begin{tikzpicture}[overlay,remember picture]
        \draw [-{Latex[round, scale=1.5]}] (-0.9,1.7) -- (0.9,1.7) 
        node [midway,above] {$A = J_b - J_b\tran$}
        node [midway,below] {\cref{e:J_JT}};
        \end{tikzpicture}
    \end{subfigure}%
    \begin{subfigure}[b]{0.18\linewidth}
        \includegraphics[width=\linewidth]{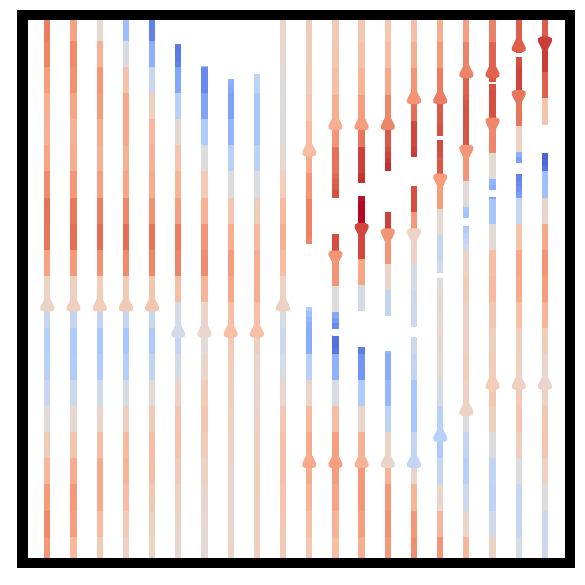}
        \caption*{$A_1$}
    \end{subfigure}%
    \begin{subfigure}[b]{0.18\linewidth}
        \includegraphics[width=\linewidth]{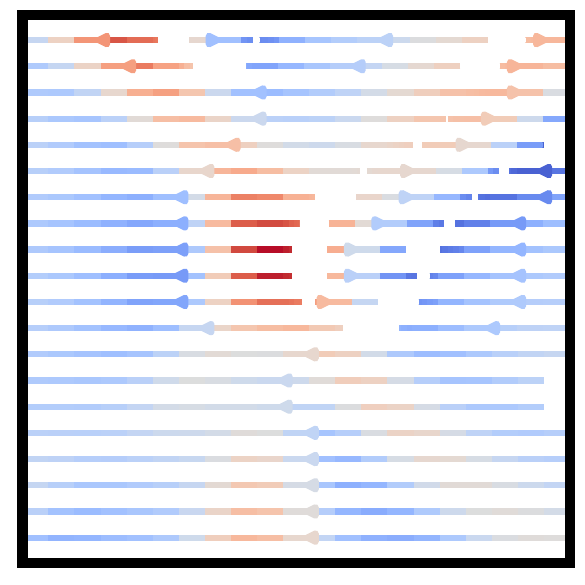}
        \caption*{$A_2$}
    \end{subfigure}%
    \begin{subfigure}[b]{0.14\linewidth}
        \centering
        \begin{tikzpicture}[overlay,remember picture]
        \draw [-{Latex[round,scale=1.5]}] (-0.9,1.7) -- (0.9,1.7) node [midway,above] {$v = \div(A)$}
        node [midway,below] {\cref{eq:div_A}};
        \end{tikzpicture}
    \end{subfigure}%
    \begin{subfigure}[b]{0.18\linewidth}
        \includegraphics[width=\linewidth]{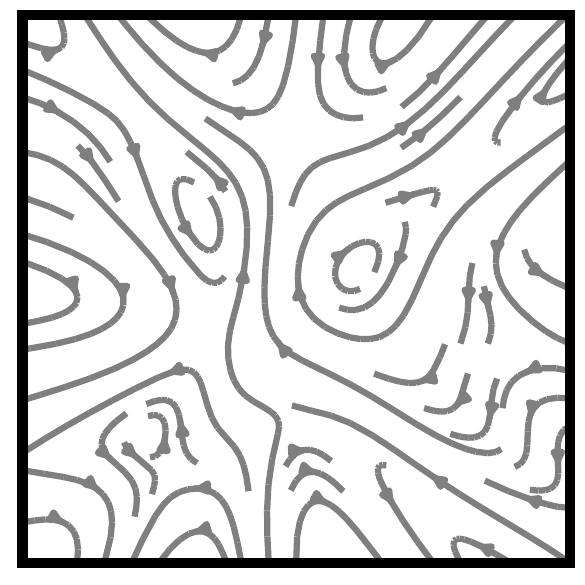}
        \caption*{$v$}
    \end{subfigure}
    \caption{Divergence-free vector fields $v:\R^d \rightarrow \R^d$ can be constructed from an antisymmetric matrix field $A:\R^d \rightarrow \R^{d\times d}$ or an arbitrary vector field $b:\R^d \rightarrow \R^d$. $J_b$ represents the Jacobian matrix of $b$, and $A_1$ and $A_2$ are the first and second rows of $A$. Color denotes divergence.}
    \label{fig:fig1}
\end{figure}

\section{Constructing divergence-free vector fields}
\label{sec:model-spec}

We will use the notation of \emph{differential forms} in $\Real^n$ for deriving the divergence-free and universality properties of our vector field constructions. 
We provide a concise overview of differential forms in \Cref{sec:prelim:k-forms-on-Rn}.
For a more extensive introduction see \eg, \citet{do1998differential,morita2001geometry}. 
Without this formalism, it is difficult to show either of these two properties; however, readers who wish to skip the derivations can jump to the matrix formulations in \eqref{eq:div_A} and \eqref{e:J_JT}.

Let us denote
$\gA^k(\Real^n)$ as the space of differential $k$-forms, $d:\gA^k(\Real^n)\too\gA^{k+1}(\Real^n)$ as the exterior derivative, and $\star:\gA^{k}(\Real^n)\too \gA^{n-k}(\Real^n)$ as the Hodge star operator that maps each $k$-form to an ($n$-$k$)-form. Identifying a vector field $v:\Real^d\too \Real^d$ as a $1$-form, $v=\sum_{i=1}^n v_i dx_i$, we can express the divergence $\div(v)$ as the composition $d \star v$. 

To parameterize a divergence-free vector field $v: \Real^n\too \Real^n$, we note that by the fundamental property of the exterior derivative, taking an arbitrary ($n$-$2$)-form $\mu\in \gA^{n-2}(\Real^n)$ we have that 
\begin{equation}
    0=d^2\mu = d(d\mu) 
\end{equation}
and since $\star$ is its own inverse up to a sign, it follows that 
\begin{equation}\label{e:star_d_mu}
    v=\star d\mu
\end{equation}
is divergence free. 
We write the parameterization $v=\star d\mu$ explicitly in coordinates.
Since a basis for $\gA^{n-2}(\Real^n)$ can be chosen to be  $\star (dx_i\wedge dx_j)$, we can write $\mu=\frac{1}{2}\sum_{i,j=1}^n\mu_{ij} \star (dx_i\wedge dx_j)$, where $\mu_{ji}=-\mu_{ij}$
Now a direct calculation shows that up to a constant sign (see Appendix-\ref{app:derivations})
\begin{equation}\label{eq:mu}
    \star d\mu = \sum_{i=1}^n \brac{\sum_{j=1}^n \frac{\partial \mu_{ij}}{\partial x_j}} dx_i.
\end{equation}
This formula is suggestive of a simple matrix formulation: If we let $A:\Real^n\too\Real^{n\times n}$ be the anti-symmetric matrix-valued function where $A_{ij}=\mu_{ij}$ then the divergence-free vector field $v=\star d\mu$ can be written as taking row-wise divergence of $A$, \ie, 
\begin{center}			
\vspace{-1em}
\colorbox{mygray} {		
\begin{minipage}{0.986\linewidth} 	
\begin{equation}\label{eq:div_A}
    v = \begin{pmatrix} \div (A_1) \\ \vdots \\ \div(A_n) \end{pmatrix}.
\end{equation}
\end{minipage}}			
\end{center}
However, this requires parameterizing $O(n^2)$ functions. A more compact representation, which starts from a vector field, can also be derived.
The idea behind this second construction is to model $\mu$ instead as $\mu=\delta \nu$, where $\nu\in\gA^{n-1}(\Real^n)$. Putting this together with \eqref{e:star_d_mu} we get that 
\begin{equation}\label{e:star_d_delta_nu}
    v = \star d \delta \nu
\end{equation}
is a divergence-free vector field. 
To provide \eqref{e:star_d_delta_nu} in matrix formulation we first write $\nu\in\gA^{n-1}(\Real^n)$ in the ($n$-$1$)-form basis, \ie, $\nu = \sum_{i=1}^n \nu_i \star dx_i$. Then a direct calculation provides up to a constant sign
\begin{equation}
    \delta \nu  = \frac{1}{2}\sum_{i,j=1}^n \brac{\frac{\partial \nu_i}{\partial x_j} - \frac{\partial \nu_j}{\partial x_i} }\star(dx_i\wedge dx_j) \label{eq:delta_nu}
\end{equation}
So, given an arbitrary vector field $b: \Real^n\too \Real^n$, and denoting $J_b$ as the Jacobian of $b$, 
we can construct $A$ as
\begin{center}			
\vspace{-1em}
\colorbox{mygray} {		
\begin{minipage}{0.986\linewidth} 	
\begin{equation}\label{e:J_JT}
    A = J_b - J_b\tran.
\end{equation}
\end{minipage}}			
\end{center}
where $J_b$ denotes the Jacobian of $b$.

To summarize, we have two constructions for divergence-free vector fields $v$: \begin{itemize}[leftmargin=70pt]
    \item[\textbf{Matrix-field}:] (equations \ref{e:star_d_mu} and \ref{eq:div_A}) $v$ is represented using an anti-symmetric matrix field $A$. 
    \item[\textbf{Vector-field}:] (equations \ref{e:star_d_delta_nu} and \ref{eq:div_A}+\ref{e:J_JT}) $v$ is represented using a vector field $b$.
\end{itemize}

\begin{wrapfigure}[9]{r}{0.3\linewidth}
\vspace{-1.4em}
\centering
\includegraphics[width=\linewidth]{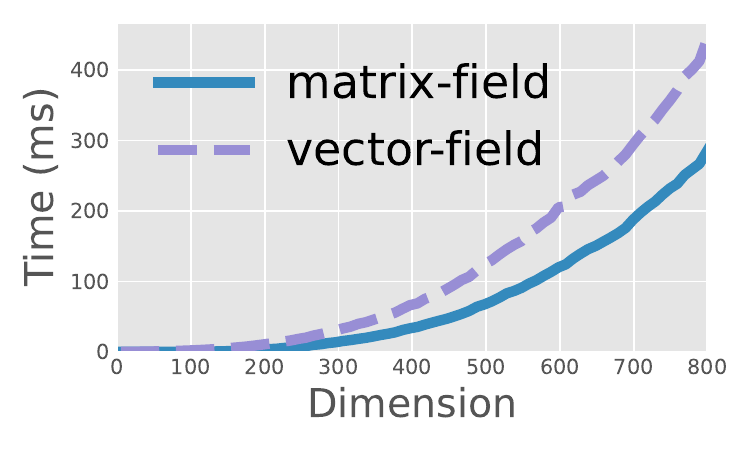}
\vspace{-1.6em}
\caption{Compute times.}
\label{fig:compute}
\end{wrapfigure}
As we show in the next section, these two approaches are maximally expressive (\ie, universal), meaning they can approximate arbitrary smooth divergence-free vector fields. However, empirically these two constructions can exhibit different practical trade-offs. The matrix-field construction has a computational advantage as it requires one less Jacobian computation, though it requires more components---$O(n^2)$ vs.~$O(n)$---to represent than the vector-field construction. This generally isn't a concern as all components can be computed in parallel. However, the vector-field construction can make it easy to bake in additional constraints; an example of this is discussed in \Cref{sec:subharmonic}, where a non-negativity constraint is imposed for modeling continuous-time probability density functions.
In \Cref{fig:compute} we show wallclock times for evaluating the divergence-free vector field based on our two constructions. Both exhibit quadratic scaling (in function evaluations) with the number of dimensions due to the row-wise divergence in \eqref{eq:div_A}, while the vector-field construction has an extra Jacobian computation.

\subsection{Universality}
Both the matrix-field and vector-field representations are universal, \ie they can model arbitrary divergence-free vector fields. The main tool in the proof is the Hodge decomposition theorem \citep{morita2001geometry,berger2003panoramic}. For simplicity we will be working on a periodic piece of $\Real^n$, namely the torus $\sT^n=[-M,M]^n/\sim$ where $\sim$ means identifying opposite edges of the $n$-dimensional cube, and $M>0$ is arbitrary. 
Vector fields with compact support can always be encapsulated in $\sT^n$ with $M>0$ sufficiently large. As $\sT^n$ is locally identical to $\Real^n$, all the previous definitions and constructions hold. 

\begin{restatable}{theorem}{universalitythm}\label{thm:universal}
The matrix and vector-field representations are universal in $\sT$, possibly only missing a constant vector field. 
\end{restatable}
A formal proof of this result is in \Cref{app:universality}.

\section{Neural Conservation Laws}
\label{sec:conservation-laws}

We now discuss a key aspect of our work, which is parameterizing solutions of scalar conservation laws. Conservation laws are a focal point of mathematical physics and have seen applications in machine learning, with the most well known examples being conserved scalar quantities often referred to as density, energy, or mass. 
Formally, a conservation law can be expressed as a first-order PDE written in  divergence form as $\tfrac{\partial \rho}{\partial t} + \div(j) = 0$,
where $j$ is known as the flux, and the divergence is taken over spatial variables. In the case of the continuity equation, there is a velocity field $u$ which describes the flow and the flux is equal to the density times the velocity field:
\begin{equation}\label{eq:conservation_law}
\frac{\partial \rho}{\partial t} + \div(\rho u) = 0
\end{equation}
where $\rho : \Real^n\too \Real^+$ and $u: \Real^n\too \Real^n$.
One can then interpret the equation to mean that $u$ transports $\rho(0,\cdot)$ to $\rho(t,\cdot)$ continuously -- without teleporting, creating or destroying mass.
Such an interpretation plays a key role in physics simulations as well as the dynamic formulation of optimal transport \citep{benamou2000computational}. 
In machine learning, the continuity equation appears in continuous normalizing flows \citep{chen2018neural}---which also have been used to approximate solutions to dynamical optimal transport. \citep{finlay2020train,tong2020trajectorynet,onken2021ot} These, however, only model the velocity $u$ and rely on numerical simulation to solve for the density $\rho$, which can be costly and time-consuming.

Instead, we observe that \eqref{eq:conservation_law} can be expressed as a divergence-free vector field by augmenting the spatial dimensions with the time variable, resulting in a vector field $v$ that takes as input $(t, x)$ and outputs $(\rho, \rho u)$. Then \eqref{eq:conservation_law} is equivalent to
\begin{equation}
    \div(v) = \div \begin{pmatrix} \rho \\ \rho u \end{pmatrix} = \frac{\partial \rho}{\partial t} + \div(\rho u) = 0
\end{equation}
where the divergence operator is now taken with respect to the joint system $(t, x)$, \ie $\tfrac{\partial}{\partial t} + \sum_{i=1}^n \tfrac{\partial}{\partial x_i}$.
We thus propose modeling solutions of conservation laws by parameterizing the divergence-free vector field $v$.
Specifically, we parameterize a divergence-free vector field $v$ and set $v_1 = \rho$ and $v_{2:n+1} = \rho u$, allowing us to recover the vector field as $u = \tfrac{v_{2:n+1}}{\rho}$, assuming $\rho \neq 0$.
This allows us to enforce the continuity equation at an architecture level. Compared to simulation-based modeling approaches, we completely forego such computationally expensive simulation procedures. Code for our experiments are available at \url{https://github.com/facebookresearch/neural-conservation-law}.

\section{Related Works}

\paragraph{Baking in constraints in deep learning}
Existing approaches to enforcing constraints in deep neural networks can induce constraints on the derivatives, such as convexity~\citep{amos2017input} or Lipschitz continuity~\citep{miyato2018spectral}. More complicated formulations involve solving numerical problems such as using solutions of convex optimization problems~\citep{amos2017optnet}, solutions of fixed-points iterations~\citep{bai2019deep}, and solutions of ordinary differential equations~\citep{chen2018neural}. These models can help provide more efficient approaches or alternative algorithms for handling downstream applications such as constructing flexible density models~\citep{chen2019residual,lu2021implicit} or approximating optimal transport paths~\citep{tong2020trajectorynet,makkuva2020optimal,onken2021ot}. However, in many cases, there is a need to solve a numerical problem in which the solution may only be approximated up to some numerical accuracy; for instance, the need to compute the density flowing through a vector field under the continuity equation~\citep{chen2018neural}.

\paragraph{Applications of differential forms} Differential forms and more generally, differential geometry, have been previously applied in manifold learning---see \eg \citet{arvanitidis_latent_2021} and \citet{bronstein2017geometric} for an in-depth overview. 
Most of the applications thus far have been restricted to 2 or 3 dimensions---either using identities like $\div \circ \text{curl}=0$ in 3D for fluid simulations \citep{rao2020physics}, or for learning geometric invariances in 2D images or 3D space \citep{Gerken2021invariance,li20213dmol}.

\paragraph{Conservation Laws in Machine Learning} \citep{sturm2022discreteconservation} previously explored discrete analogs of conservation laws by conserving mass via a balancing operation in the last layer of a neural network. \citep{muller2022noetherconservation} utilizes a wonderful application of Noether's theorem to model conservation laws by enforcing symmetries in a Lagrangian represented by a neural network. 

\section{Neural approximations to PDE solutions}
\label{sec:fluids}
As a demonstration of our method, we apply it to neural-based PDE simulations of fluid dynamics. 
First, we apply it to modelling inviscid fluid flow in the open ball $\mathbb{B} \subseteq \R^3$ with free slip boundary conditions, then to a 2d example on the flat Torus $\mathbb{T}^2$, but with more complex initial conditions. 
While these are toy examples, they demonstrate the value of our method in comparison to existing approaches---namely that we can exactly satisfy the continuity equation and preserve exact mass.

\paragraph{The Euler equations of incompressible flow}
\label{sec:fluids:euler-intro}
The incompressible Euler equations \citep{feynman_1989} form an idealized model of inviscid fluid flow, governed by the system of partial differential equations\footnote{The convective derivative appearing in \eqref{eq:fluid_eqs}, $\nabla_u u (x) = \lim_{h \to 0} \tfrac{u(x + hu(x)) - u(x)}{h} = [Du](u)$
is also often written as $(\nabla \cdot u) u$.}
\begin{equation}\label{eq:fluid_eqs}
    \tfrac{\partial \rho}{\partial t} + \div(\rho u) = 0,
    \quad\quad\quad
    \tfrac{\partial u}{\partial t} + \nabla_u u = \tfrac{\nabla p}{\rho},
    \quad\quad
    \div(u) = 0
\end{equation}
in three unknowns: the fluid velocity $u(t,x) \in \R^3$, the pressure $p(t,x)$, and the fluid density $\rho(t,x)$. While the fluid velocity and density are usually given at $t=0$, the initial pressure is not required. Typically, on a bounded domain $\Omega \subseteq \R^n$, these are supplemented by the \textit{free-slip} boundary condition and initial conditions
\begin{gather} u \cdot n = 0 \text{ on } \partial \Omega   \qquad u(0,x) = u_0 \text{ and } \rho(0,x) = \rho_0 \qquad \text{ on } \Omega \label{eq:inc_euler_bc_init}\end{gather}
The density $\rho$ plays a critical role since in addition to being a conserved quantity, it influences the dynamics of the fluid evolution over time. In numerical simulations, satisfying the continuity equation as closely as possible is desirable since the equations in (\ref{eq:fluid_eqs}) are coupled. Error in the density feeds into error in the velocity and then back into the density over time.  In the finite element literature, a great deal of effort has been put towards developing conservative schemes that preserve mass (or energy in the more general compressible case)---see \citet{guermond_projection_2000} and the introduction of \citet{almgren_conservative_1998} for an overview. 
But since the application of physics informed neural networks (PINNs) to fluid problems is much newer, 
conservative constraints have only been incorporated as penalty terms into the loss \citep{mao_physics-informed_2020,jin_nsfnets_2021}. 

\subsection{Physics informed neural networks}
\label{sec:fluids:PINN}
Physics Informed Neural Networks (PINNs; \citet{raissi_physics-informed_2019,raissi_physics_2017}) have recently received renewed attention as an application of deep neural networks. While using neural networks as approximate solutions to PDE had been previously explored (e.g in \cite{lagaris1998artificial}), modern advances in automatic differentiation algorithms have made the application to much more complex problems feasible \citep{raissi_physics-informed_2019}. 
The ``physics'' in the name is derived from the incorporation of \textit{physical} terms into the loss function, which consist of adding the squared residual norm of a PDE. 
For example, to train a neural network $\phi = [\rho,p,u]$ to satisfy the Euler equations, the standard choice of loss to fit to is
\begin{alignat*}{4}
&L_{F} &&= \norm{u_t + \nabla_u u + \tfrac{\nabla p}{\rho}}_{\Omega}^2 
\quad &&L_{\div} = \norm{\div(u)}_{\Omega} 
\quad &&L_{I} = \norm{u(0,\cdot) - u_0(\cdot)}_{\Omega}^2 + \norm{\rho(0,\cdot) - \rho_0(\cdot) }_{\Omega}\\
&L_{\text{Cont}} &&= \norm{\tfrac{\partial \rho}{\partial t} + \div(\rho u) }^2_\Omega &&L_{G} = \norm{u\cdot n}_{\partial \Omega}^2 &&L_{\text{total}} = \gamma \cdot [L_F,L_I,L_{\div},L_\text{Cont}, L_{G}]
\end{alignat*}
where $\gamma=(\gamma_F,\gamma_I,\gamma_\div \gamma_{Cont}, \gamma_G)$ denotes suitable coefficients (hyperparameters). The loss term $L_G$ ensures fluid does not pass through boundaries, when they are present.
Similar approaches were taken in \citep{mao_physics-informed_2020} and \citep{jagtap_conservative_2020} for related equations. While schemes of this nature are very easy to implement, they have the drawback that since PDE terms are only penalized and not strictly enforced, one cannot make guarantees as to the properties of the solution. 

To showcase the ability of our method to model conservation laws, we will parameterize the density and vector field as $v = [\rho, \rho u]$, as detailed in \Cref{sec:conservation-laws}.
This means we can omit the term $L_\text{Cont}$ as described in \Cref{sec:fluids:PINN} from the training loss. 
The divergence penalty, $L_{\div}$ remains when modeling incompressible fluids, since $u$ is not necessarily itself divergence-free -- it is $v= [\rho,\rho u]$ which is divergence free. 
In order to stablize training, we can modify the loss terms
$L_F,L_G,L_I$ to avoid division by $\rho$. This is detailed in \Cref{app:stable_losses}.

\begin{figure}
    \centering
    \begin{subfigure}[b]{0.48\linewidth}
    \begin{subfigure}[b]{\linewidth}
    \includegraphics[width=\linewidth]{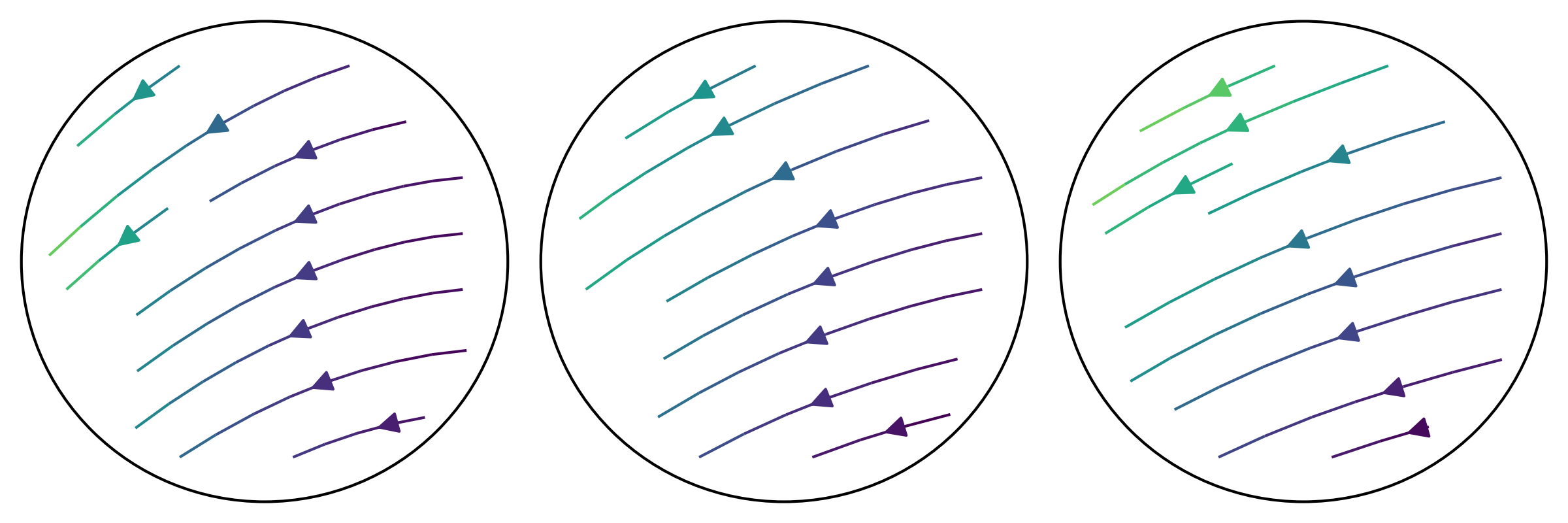}\\
    \includegraphics[width=\linewidth]{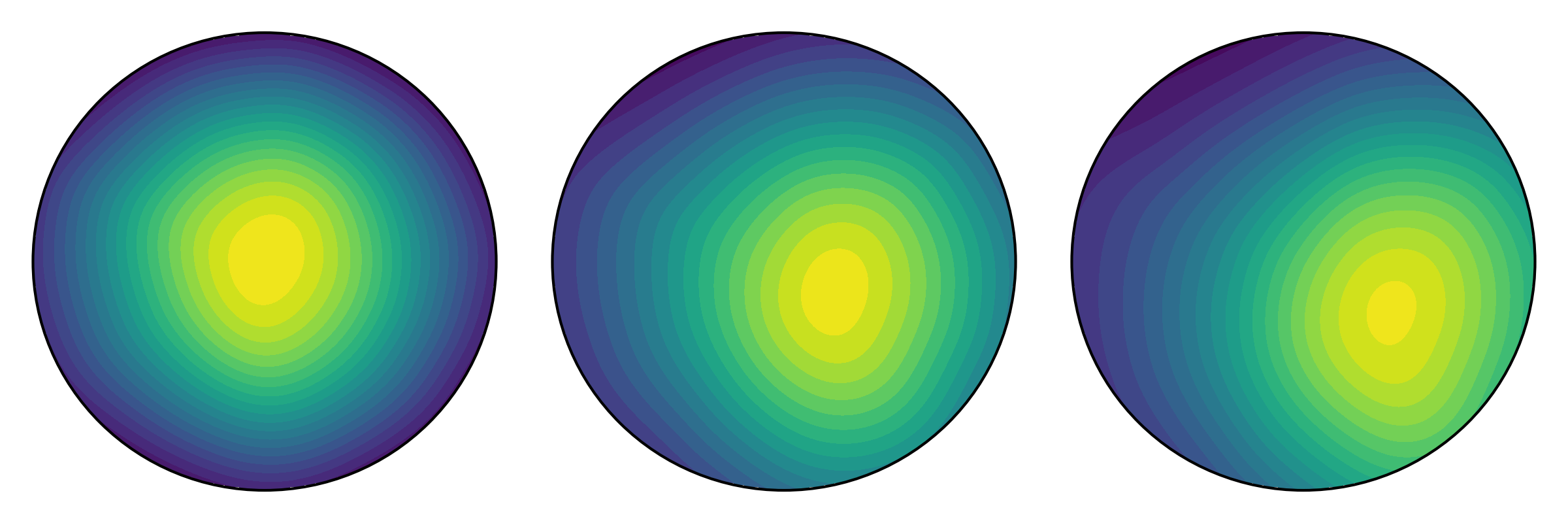}
    \begin{center}
        $t$ = 0.00 \quad\quad\;\;
        $t$ = 0.25 \quad\quad\;\;
        $t$ = 0.50
    \end{center}
    \subcaption{Incompressible baseline}
    \end{subfigure}
    \end{subfigure}
    \begin{subfigure}[b]{0.48\linewidth}
    \begin{subfigure}[b]{\linewidth}
    \includegraphics[width=\linewidth]{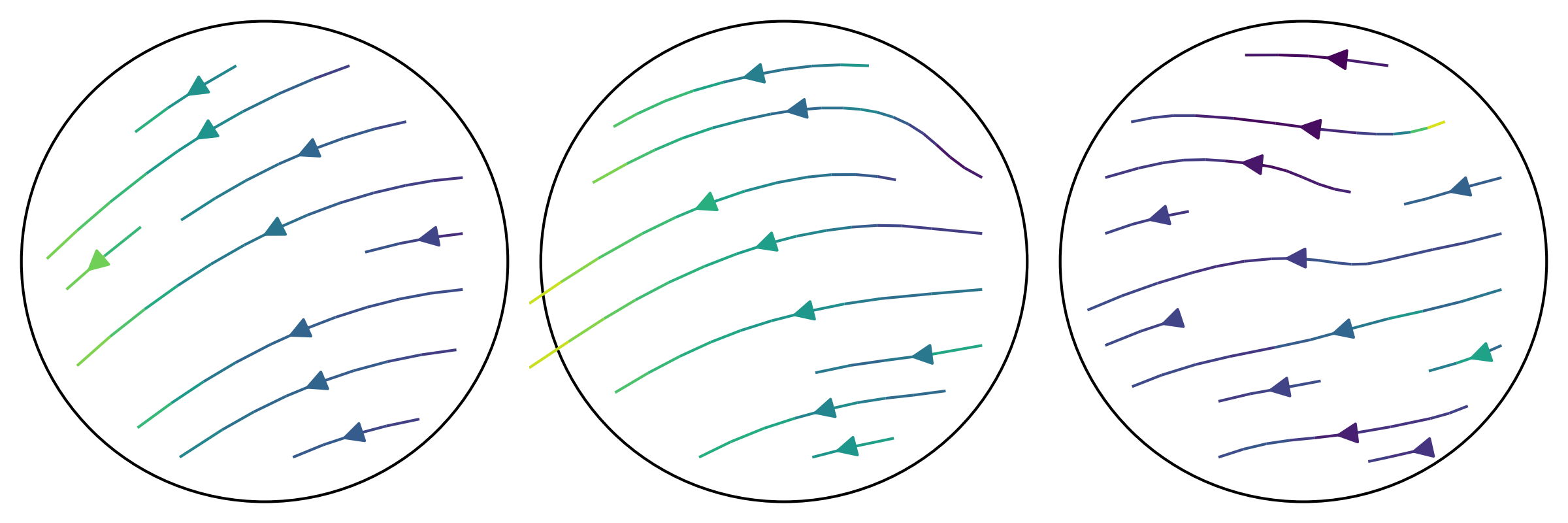}
    \includegraphics[width=\linewidth]{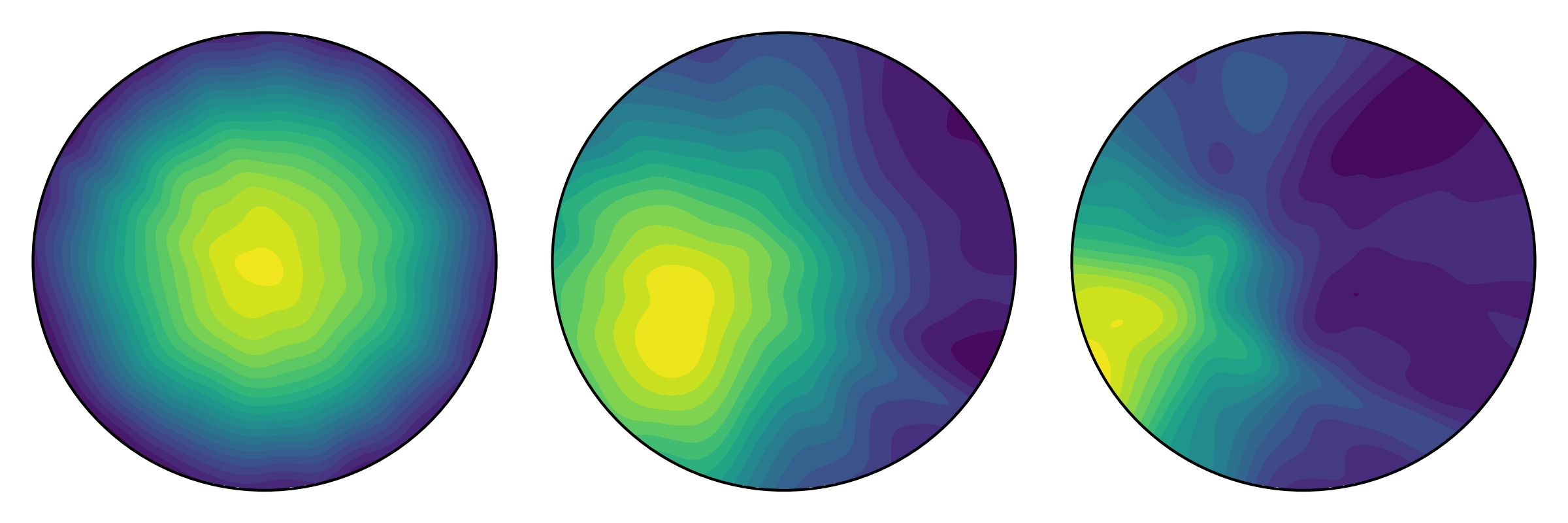}
    \begin{center}
        $t$ = 0.00 \quad\quad\;\;
        $t$ = 0.25 \quad\quad\;\;
        $t$ = 0.50
    \end{center}
    \subcaption{Neural Conservation Law (Ours)}
    \end{subfigure}
    \end{subfigure}
\caption{Exactly satisfying the continuity equation yields much better performance than only satisfying the incompressible (divergence-free) condition of the vector field, even though both approaches obtain very small loss values. 2D slices of the learned solution of $u$ and $\rho$ are shown.}
\label{fig:euler_simulation_plot_3d}
\end{figure}

\subsection{Incompressible variable density inside the 3D unit ball}
\label{sec:fluids:3d-experiments}
We first construct a simple example within a bounded domain, specifically, we will consider the Euler equations inside $B(0,1) \subseteq \R^3$, with the initial conditions
\begin{equation} 
\rho(0, x) = 3/2 - \norm{x}^2 \qquad v(0, x) =( -2, x_0-1, 1/2)
\end{equation}
As a baseline, we compare against a parameterization that uses the divergence-free vector field only for the incompressible constraint. This is equivalent to the curl approach described in \cite{rao2020physics}; another neural network is used to parameterize $\rho$. In contrast, our generalized formulation of divergence-free fields allow us to model the continuity equation exactly, as in this case, $[\rho, \rho u]$ is 4 dimensional.
We parameterize our method using $v = [\rho, \rho u]$, as detailed in \Cref{sec:conservation-laws}, with the vector-field approach from \Cref{sec:model-spec}. 


The results for training both networks are shown in \Cref{fig:euler_simulation_plot_3d}. Surprisingly, even though both architectures signficantly improve the training loss after a comparable number of iterations (see \Cref{apx:details}), the curl approach fails to advect the density accurately, whereas our NCL approach evolves the density as it should given the initial conditions.

\subsection{Incompressible variable density flow on $\mathbb{T}^2$}
\label{sec:fluids:manifold-experiments}
As a more involved experiment, we model the Euler equations on the flat 2-Torus $\mathbb{T}^2$. This is parameterized by a multilayer perceptron pre-composed with the periodic embedding $\iota : [0,1]^2 \to \mathbb{T}^2$
\begin{gather*} \iota(x,y) = \left[  \cos(2\pi x ), \sin(2\pi x),\cos(2\pi y ),\sin(2\pi y) \right]
\end{gather*}
The continuity equation is enforced exactly, using the method detailed in \Cref{sec:conservation-laws}, so the $L_{\text{Cont}}$ term is omitted from $L_\text{total}$. The initial velocity and density are
\[ \rho_0(x,y) = (z_1 + z_3)^2 + 1 \qquad v_0(x,y) = [e^{z_3}, e^{z_1}/2] \]
where $z = \theta(x,y,1)$.
Here we used the matrix-field parameterization as referred to in \Cref{sec:model-spec}, for the simple reason that it performed better empirically.
We also include the harmonic component from Hodge decomposition, which is just a constant vector field, to ensure we have full universality.
Full methodology for the experiment is detailed in \Cref{apx:details}. 
The results of evolving the flow for time $t \in [0, \nicefrac{1}{3}]$ are shown in \Cref{fig:fluids:2d-tori-experiment}, as well as a comparison with a reference finite element solution generated by the conservative splitting method from \citep{guermond_projection_2000}. In general, our model fits the initial conditions well and results in fluid-like movement of the density, however small approximation errors can accumulate from simulating over a longer time horizon. To highlight this, we also plotted the best result we were able to achieve with a standard PINN. Although the PINN made progress minimizing the loss, it was unable to properly model the dynamics of the fluid evolution. While this failure can likely be corrected by tuning the $\gamma$ constant (see \Cref{sec:fluids:PINN}), we were unable to find a combination that fit the equation correctly. 
By contrast, our NCL model worked with the first configuration we tried.


\begin{figure}
\begin{subfigure}[b]{0.55\linewidth}
\centering
\begin{subfigure}[b]{0.05\linewidth}
    \rotatebox[origin=c]{90}{PINN}\vspace{5mm}
\end{subfigure}%
\begin{subfigure}[b]{0.95\linewidth}
    \includegraphics[width=\linewidth]{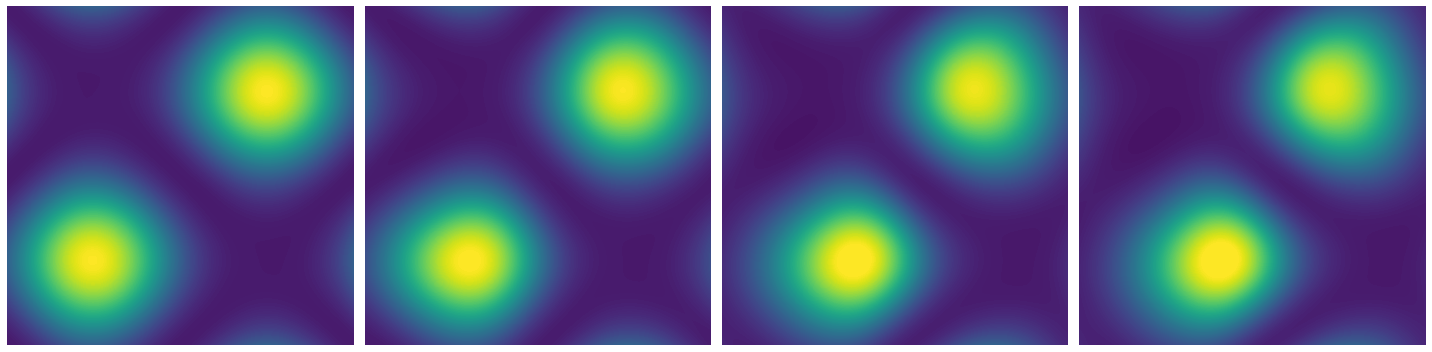}
\end{subfigure}\\
\begin{subfigure}[b]{0.05\linewidth}
    \rotatebox[origin=c]{90}{NCL}\vspace{5.7mm}
\end{subfigure}%
\begin{subfigure}[b]{0.95\linewidth}
    \includegraphics[width=\linewidth]{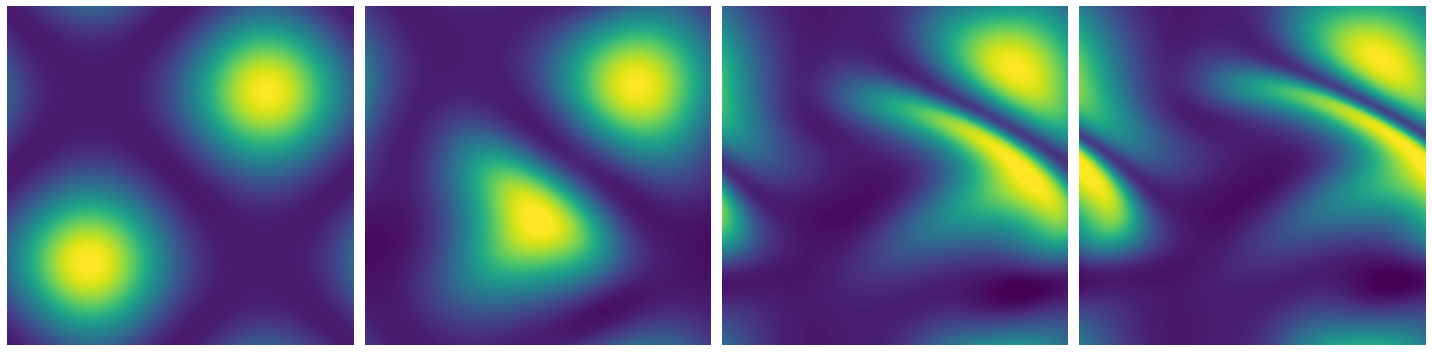}
\end{subfigure}\\
\begin{subfigure}[b]{0.05\linewidth}
    \rotatebox[origin=c]{90}{FEM}\vspace{5.7mm}
\end{subfigure}%
\begin{subfigure}[b]{0.95\linewidth}
    \includegraphics[width=\linewidth]{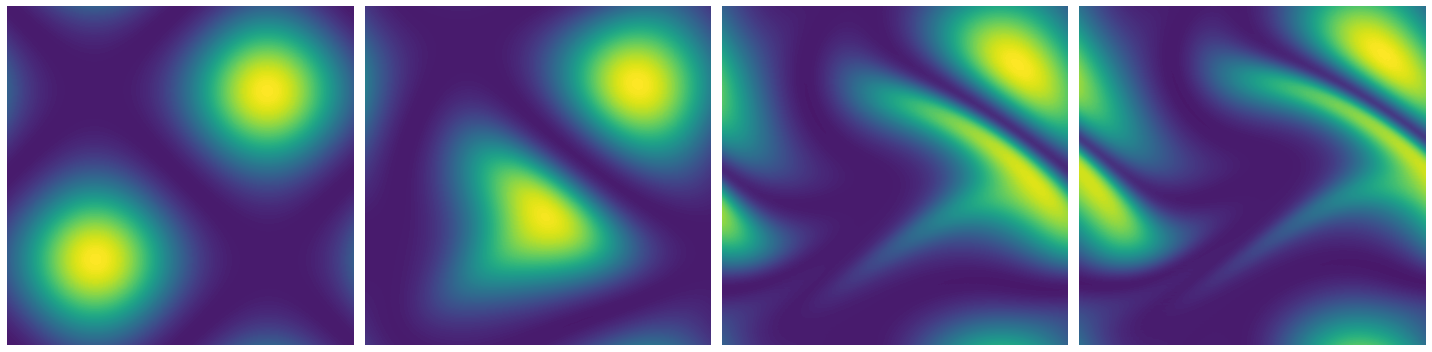}
\end{subfigure}\\
\hphantom{ttt} $t$=0.0 \hspace{2.3em} $t$=0.1 \hspace{2.3em} $t$=0.25 \hspace{2.3em} $t$=0.3
\caption{Learned density function}
\end{subfigure}%
\begin{minipage}[b]{0.45\linewidth}
\begin{subfigure}[b]{\linewidth}
\centering
\includegraphics[width=0.93\linewidth]{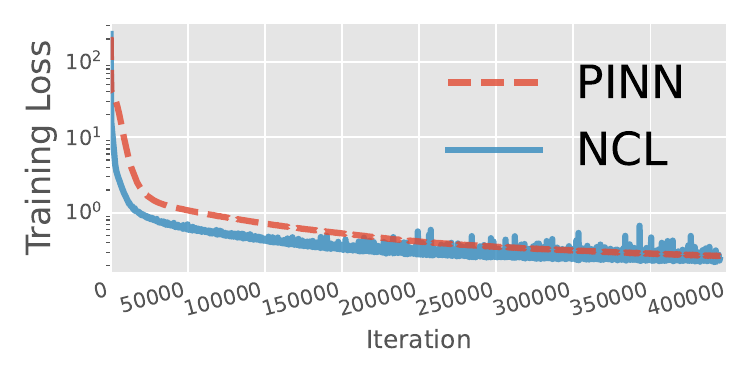}
\vspace{-0.6em}
\caption{Convergence of training loss}
\end{subfigure}\\
\begin{subfigure}[b]{\linewidth}
\centering
\includegraphics[width=0.93\linewidth]{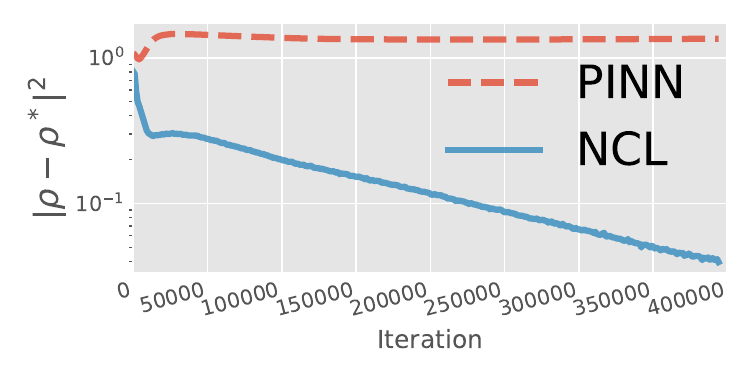}
\vspace{-0.6em}
\caption{Comparison to FEM solution}
\end{subfigure}
\end{minipage}
\caption{While both PINN and NCL models minimize the training loss effectively and fit the initial conditions, the PINN fails to learn the dynamics of the advected density. When compared to a gold standard FEM solution, our NCL model nicely exhibits linear convergence to the solution.}
\label{fig:fluids:2d-tori-experiment}
\end{figure}

\section{Learning the Hodge decomposition}

Using our divergence-free models, we can also learn the Hodge decomposition of a vector field itself. Informally, the Hodge decomposition states that any vector field can be expressed as the combination of a gradient field (\ie gradient of a scalar function) and a divergence-free field. 
There are many scenarios where a neural network modelling a vector field should be a gradient field (theoretically), such as score matching~\citep{song2019generative}, Regularization by Denoising~\citep{hurault2021gradient}, amortized optimization \citep{xue20a}, and meta-learned optimization~\citep{andrychowicz2016learning}.
Many works, however, have reported that parameterizing such a field as the gradient of a potential hinders performance, and as such they opt to use an unconstrained vector field. Learning the Hodge decomposition can then be used to remove any extra divergence-free, \ie rotational, components of the unconstrained model.

As a demonstration of the ability of our model to scale in dimension, we construct a synthetic experiment where we have access to the ground truth Hodge decomposition. 
This provides a benchmark to validate the accuracy of our method.
To do so, we construct the target field $\hat v$ as
\begin{equation}
    \hat v = \nabla w + \eta 
\end{equation}
where $\eta$ is a fixed divergence-free neural network, and $w$ is a fixed scalar network. We use the same embedding from \Cref{sec:fluids:manifold-experiments}, but extended to $\mathbb{T}^n$, for $n \in \{25, 50, 100\}$. We then parameterize another divergence-free model $v_\theta$ which has a different architecture than the target $\eta$. Inspired by the Hodge decomposition, we propose using the following objective for training:
\begin{equation} \label{eq:hodge_loss}
\ell(\theta) = \Vert J_{\hat v - v_\theta} - [J_{\hat v - v_\theta}]\tran \Vert_F
\end{equation}
where $J_{\hat v - v_\theta}$ denotes the Jacobian of $\hat v - v_\theta$.
In the notation of differential forms, this corresponds to minimizing $\Vert \delta(\hat v- v_\theta) \Vert_F$, which when zero implies
$v_\theta = \hat v + \delta \mu$, where $\mu$ is an arbitrary $n$-form. However, since $v_\theta$ is divergence-free, it must hold $\delta \mu = - \nabla w$,
\ie the left-over portion of $\hat{v} - v_\theta$ is exactly $\nabla w$, because the Hodge decomposition is an orthogonal direct sum \citep{warner1983foundations}.

\begin{wrapfigure}[8]{r}{0.35\linewidth}
\vspace{-1.5em}
\centering
\includegraphics[width=\linewidth]{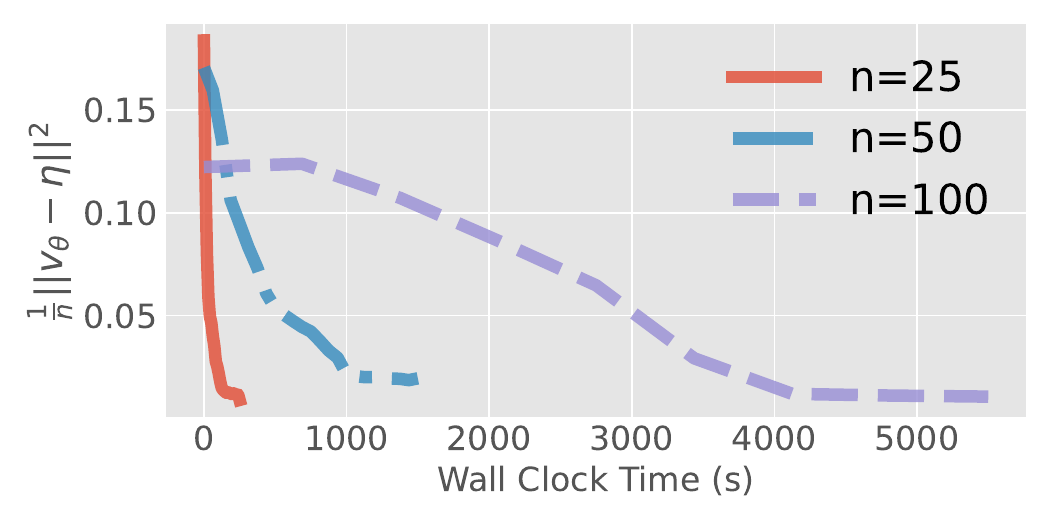}
\vspace{-1.5em}
\caption{$L_2$ to ground truth.}
\label{fig:hodge}
\end{wrapfigure}
For evaluation, we plot the squared $L_2$ norm to the ground truth $\eta$ against wall clock time. 
While this demonstration is admittedly artificial, it serves to show that our model scales efficiently to moderately high dimensions -- significantly beyond those achievable using grid based discretizations of divergence-free fields (see \cite{Bhatia2013survey} for an overview of such methods). 
The results are shown in \Cref{fig:hodge}. 
While larger dimensions are more costly, we are still able to recover the ground truth in dimensions up to 100.

\section{Dynamical optimal transport}

The calculation of optimal transport (OT) maps between two densities of interest $p_0$ and $p_1$ serves as a powerful tool for computing distances between distributions. 
Moreover, the inclusion of OT into machine learning has inspired multiple works on building efficient or stable generative models~\citep{arjovsky2017wasserstein,tong2020trajectorynet,makkuva2020optimal,huang2020convex,finlay2020train,rout2021generative,rozen2021moser}. In particular, we are interested in the dynamical formulation of \citet{benamou2000computational}, which is a specific instance where the transport map is characterized through a vector field. Specifically, we wish to solve the optimization problem
\begin{equation}\label{eq:dynamical_ot}
    \min_{\rho, u}\;\; \int_0^1 \int_\gM \|u(t, x)\|^2 \rho(t, x) \;dx dt
\end{equation}
subject to the terminal constraints, $\rho(0, \cdot) = p_0(\cdot)$ and $\rho(1, \cdot) = p_1(\cdot)$, as well as remaining non-negative $\rho \geq 0$ on the domain and satisfying the continuity equation $\tfrac{\partial\rho}{\partial t} + \div(\rho u) = 0$. 

In many existing cases, the models for estimating either $\rho$ and \eqref{eq:dynamical_ot} are completely decoupled from the generative model or transport map~\citep{arjovsky2017wasserstein,makkuva2020optimal}, or in some cases $\rho$ would be computed through numerical simulation and \eqref{eq:dynamical_ot} estimated through samples~\citep{chen2018neural,finlay2020train,tong2020trajectorynet}. 
These approaches may not be ideal as they require an auxiliary process to satisfy or approximate the continuity equation.

In contrast, our NCL approach allows us to have access to a pair of $\rho$ and $u$ that always satisfies the continuity equation, without the need for any numerical simulation. 
As such, we only need to match the terminal conditions, and will automatically have access to a vector field $u$ that transports mass from $\rho(0, \cdot)$ to $\rho(1, \cdot)$. 
In order to find an optimal vector field, we can then optimize $u$ with respect to the objective in \eqref{eq:dynamical_ot}.

\subsection{Parameterizing non-negative densities with subharmonic functions} 
\label{sec:subharmonic}

In this setting, it is important that our density is non-negative for the entire domain and within the time interval between 0 and 1. Instead of adding another loss function to satisfy this condition, we show how we can easily bake this condition into our model through the vector-field parameterization.

Recall for the vector-field parameterization of the continuity equation, we have a vector field $b_\theta(y)$ with $y = [t, x]$, for $t \in [0, 1]$, $x \in \R^n$. This results in the parameterization 
\begin{equation}
    \rho = \div \left( \tfrac{\partial b_1}{\partial x} - \tfrac{\partial b_2:n+1}{\partial t}\tran{} \right) \quad \text{ for } i=1, \dots, n+1.
\end{equation}
We now make the adjustment that $\tfrac{\partial b_{2:n+1}}{\partial t}$ = $0$. 
This condition does not reduce the expressivity of the formulation, as $\rho$ can still represent any density function.
Interestingly, this modification does imply that $\rho = \nabla^2 b_1$ with $\nabla^2$ being the Laplacian operator, so that satisfying the boundary conditions is equivalent to solving a time-dependent Poisson's equation $\nabla^2 b_1(t, \cdot) = p_t$. 

Furthermore, the class of \emph{subharmonic} functions is exactly the set of functions $g$ that satisfies $\nabla^2 g(0, \cdot) \geq 0$. In one dimension, this corresponds to the class of convex functions but is a strict generalization of convex functions in higher dimensions. With this in mind, we propose using the parameterization
\begin{equation}\label{eq:subharmonic}
    b_1(t, x) = \sum_{k=1}^K w_k(t) \phi(a_k(t) x + b_k(t)) \quad \text{ where } \phi(z) = \tfrac{1}{4\pi} \left( \log(\|z \|^2) + E_1( \nicefrac{\|z\|^2}{4}) \right)
\end{equation}
with $E_1$ being the exponential integral function, and $w_k, a_k \in \R^+$ and $b_k \in \R^n$ are free functions of $t$. This a weighted sum of generalized linear models, where the choice of nonlinear function $\phi$ is chosen because it is the exact solution to the Poisson's equation for a standard normal distribution when $n = 2$, \ie $\nabla^2 \phi = \mathcal{N}(0, I)$, while for $n \geq 2$ this remains a subharmonic function. As such, this can be seen as a generalization of a mixture model, which are often referenced to be universal density approximators, albeit requiring many mixture components \citep{goodfellow2016deep}. 

\begin{figure}
    \centering
    \begin{subfigure}[b]{0.49\linewidth}
         \centering
         \includegraphics[width=\linewidth]{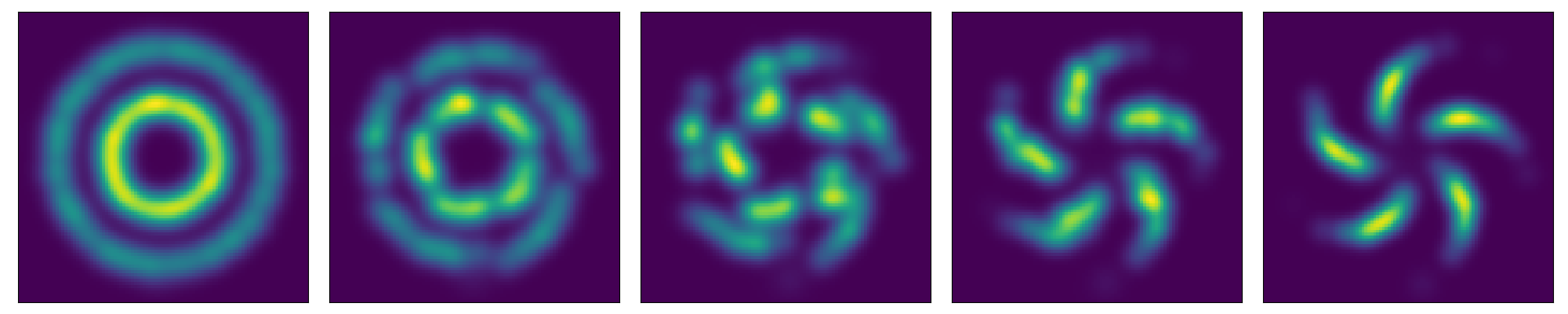}\\
         \includegraphics[width=\linewidth]{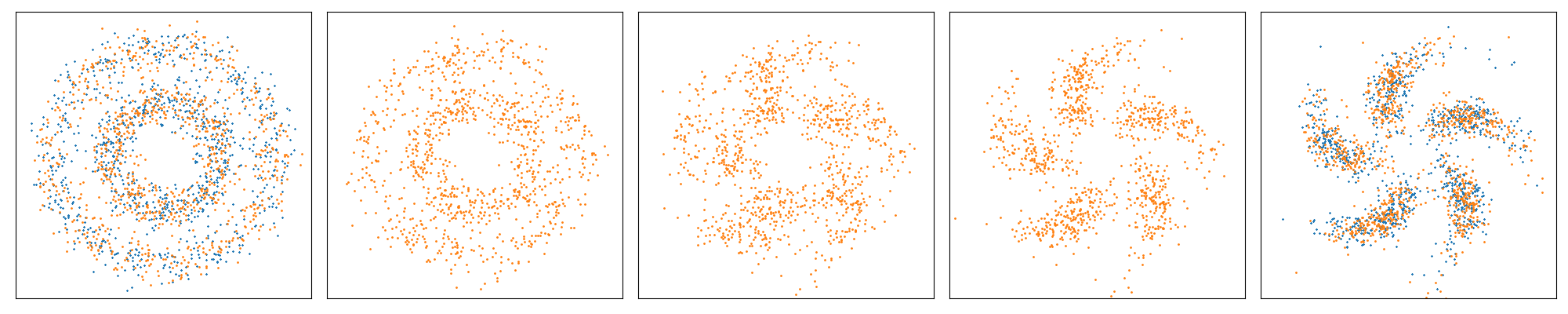}\\
         \includegraphics[width=\linewidth]{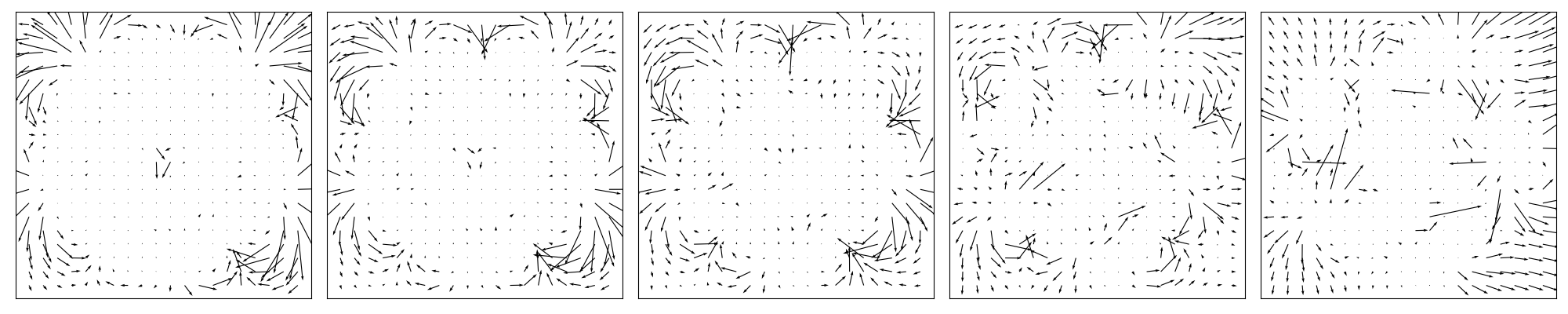}
         \caption{Circles $\longleftrightarrow$ Pinwheel}
     \end{subfigure}
     \begin{subfigure}[b]{0.49\linewidth}
         \centering
         \includegraphics[width=\linewidth]{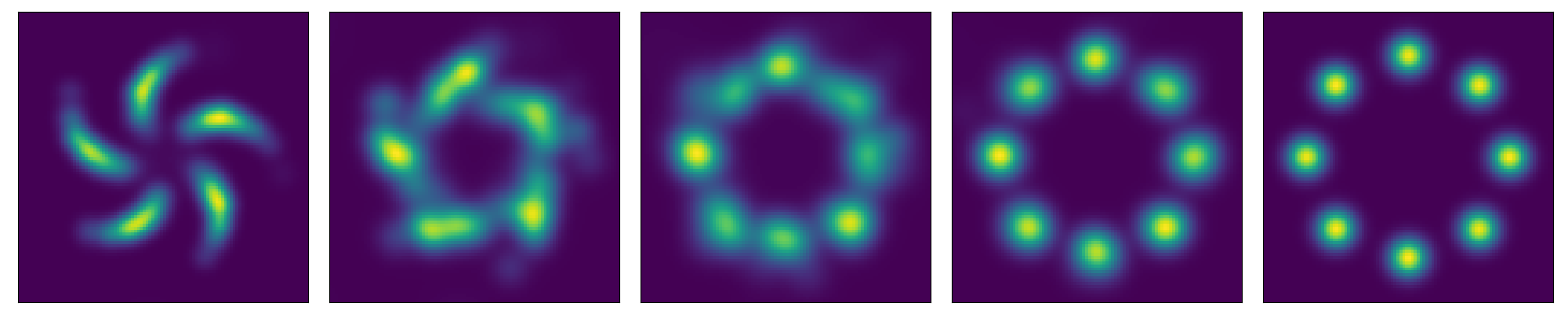}\\
         \includegraphics[width=\linewidth]{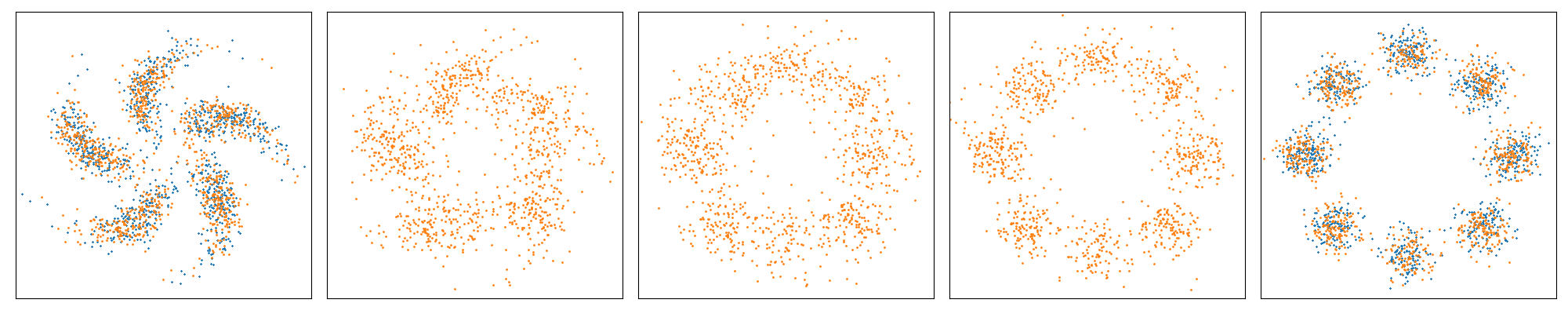}\\
         \includegraphics[width=\linewidth]{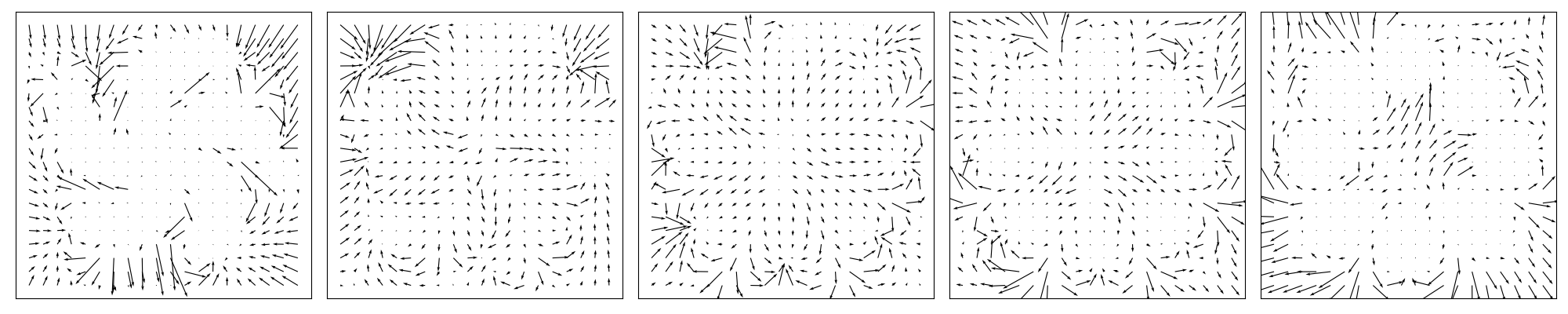}
         \caption{Pinwheel $\longleftrightarrow$ 8Gaussians}
     \end{subfigure}
    \caption{Learned approximation to the dynamical optimal transport map. (\textit{top}) Density $\rho$. (\textit{mid}) Transformed samples in orange, with samples from $p_0$ and $p_1$ in blue. (\textit{bottom}) Learned vector field.}
    \label{fig:ot_paths}
\end{figure}

\subsection{Experiments}

We experiment with pairs of 2D densities (visualized in \Cref{fig:ot_paths}) and fit a pair of $(\rho, u)$ to \eqref{eq:dynamical_ot} while satisfying boundary conditions. Specifically, we train with the loss
\begin{equation}
    \min_{\rho, u} \;
    \lambda \E_{x \sim \tilde{p}_0} \left[\left| \rho(0, x) - p_0(x) \right|\right]
    + \lambda \E_{x \sim \tilde{p}_1} \left[\left| \rho(1, x) - p_1(x) \right|\right] 
    + \int_0^1 \int_\gM \|u(t, x)\|^2 \rho(t, x) \;dx dt
\end{equation}
where $\tilde{p}_i$ is a mixture between $p_i$ and a uniform density over a sufficiently large area, for $i=0,1$, and $\lambda$ is a hyperparameter. We use a batch size of 256 and set $K$=128 (from \eqref{eq:subharmonic}). Qualitative results are displayed in \Cref{fig:ot_paths}, which include interpolations between pairs of densities.

Furthermore, we can use the trained models to estimate the Wasserstein distance, by transforming samples from $p_0$ through the learned vector field to time $t$=1. We do this for 5000 samples and compare the estimated value to (i) one estimated by that of a neural network trained through a minimax optimization formulation \citep{makkuva2020optimal}, and (ii) one estimated by that of a discrete OT solver \citep{bonneel2011displacement} based on kernel density approximation and interfaced through the \texttt{pot} library \citep{flamary2021pot}.

Estimated squared Wasserstein distance as a function of wallclock time are shown in \Cref{fig:w2_plots}.
We see that our estimated values roughly agree with the discrete OT algorithm.
However, the baseline minimax approach consistently underestimates the optimal transport distance. This is an issue with the minimax formulation not being able to cover all of the modes, and has been commonly observed in the literature (\eg 
\citet{salimans2016improved,che2016mode,srivastava2017veegan}). Moreover, in order to stabilize minimax optimization in practice, \citet{makkuva2020optimal} made use of careful optimization tuning such as reducing momentum \citep{radford2015unsupervised}. In contrast, our NCL approach is a simple optimization problem.

\begin{figure}
    \centering
    \begin{subfigure}[b]{0.33\linewidth}
    \includegraphics[width=\linewidth]{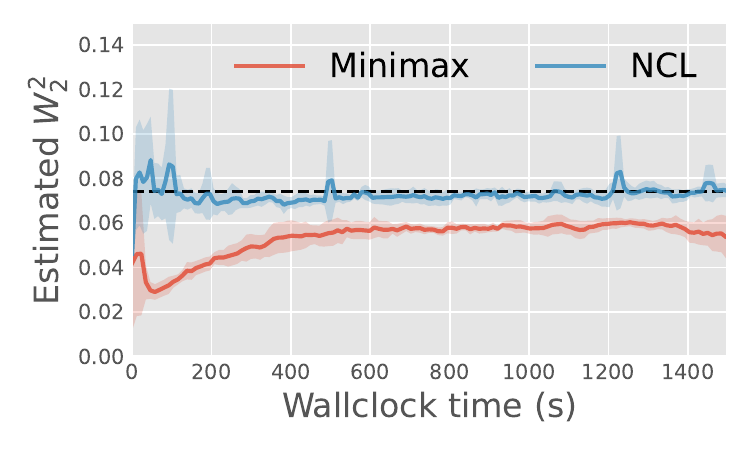}
    \vspace{-2em}
    \caption{Circles $\leftrightarrow$ Pinwheel}
    \end{subfigure}%
    \begin{subfigure}[b]{0.33\linewidth}
    \includegraphics[width=\linewidth]{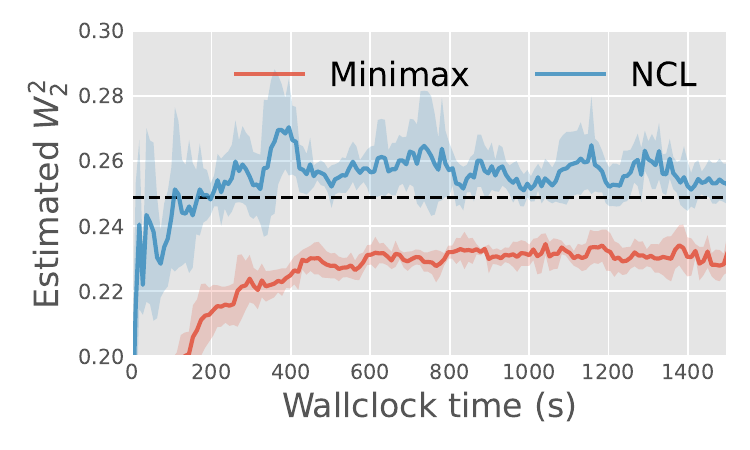}
    \vspace{-2em}
    \caption{Pinwheel $\leftrightarrow$ 8Gaussians}
    \end{subfigure}%
    \begin{subfigure}[b]{0.33\linewidth}
    \includegraphics[width=\linewidth]{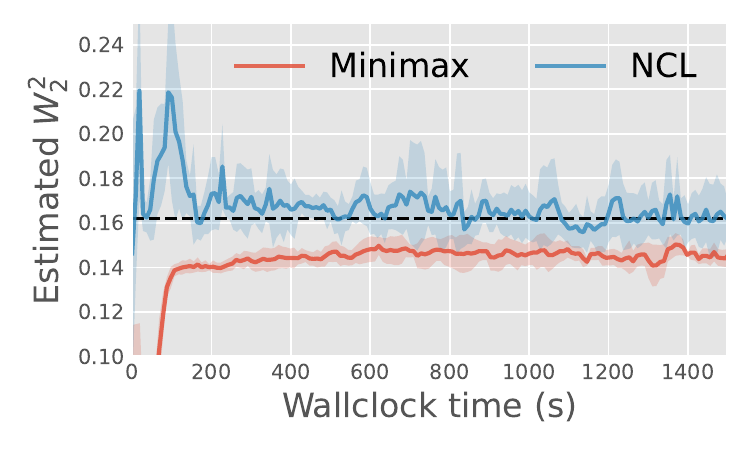}
    \vspace{-2em}
    \caption{8Gaussians $\leftrightarrow$ Circles}
    \end{subfigure}
    \caption{Neural Conservation Laws on approximating dynamical optimal transport converges in minutes. The alternative minimax formulation tends to underestimate the Wasserstein distance. Dashed line is an estimate from a discrete OT algorithm. Shaded regions denote standard deviation.}
    \label{fig:w2_plots}
\end{figure}

\section{Conclusion}

We proposed two constructions for building divergence-free neural networks from an original unconstrained smooth neural network, where either an original matrix field or vector field is parameterized. We found both to be useful in practice: the matrix field approach is generally more flexible, while the vector field approach allows the addition of extra constraints such as a non-negativity constraint. We combined these methods with the insight that the continuity equation can reformulated as a divergence-free vector field, and this resulted in a method for building deep neural networks that are constrained to always satisfy the continuity equation, which we refer to as Neural Conservation Laws. 

Currently, these models are difficult to scale due to extensive use of automatic differentiation for computing divergence of neural networks. It may be possible to combine our approach with parameterizations that provide cheap divergence computations~\citep{chen2019neural}. Nevertheless, the development of efficient batching and automatic differentiation tools is an active research area, so we expect these tools to improve as development progresses. 

\section*{Acknowledgements} 
We acknowledge the Python community
\citep{van1995python,oliphant2007python}
and the core set of tools that enabled this work, including
PyTorch \citep{paszke2019pytorch},
functorch \citep{functorch2021},
torchdiffeq \citep{torchdiffeq},
\textsc{Jax} \citep{jax2018github},
Flax \citep{flax2020github},
Hydra \citep{Yadan2019Hydra},
Jupyter \citep{kluyver2016jupyter},
Matplotlib \citep{hunter2007matplotlib},
numpy \citep{oliphant2006guide,van2011numpy}, and
SciPy \citep{jones2014scipy}. Jack Richter-Powell would also like to thank David Duvenaud and the Vector Institute for graciously supporting them over the last year.

\bibliographystyle{abbrvnat}
\bibliography{references}

\newpage

\appendix

\section{Preliminaries: Differential forms in $\Real^n$}
\label{sec:prelim:k-forms-on-Rn}

We provide an in-depth explanation of our divergence-free construction in this section. 

We will use the notation of \emph{differential forms} in $\Real^n$ that will help derivations and proofs in the paper. Below we provide basic definitions and properties of differential forms, for more extensive introduction see \eg, \citep{do1998differential,morita2001geometry}.

We let $x=(x_1,\ldots,x_n)\in\Real^n$, and $dx_1,\ldots,dx_n$ the coordinate differentials, \ie, $dx_i(x)=x_i$ for all $i\in [n]=\set{1,\ldots,n}$. 
The linear vector space of $k$-forms in $\Real^n$, denoted $\Lambda^k(\Real^n)$, is the space of $k$-linear alternating maps
\begin{equation}
    \varphi: \overbrace{\Real^n \times \cdots \times \Real^n}^{k\text{ times}} \too \Real
\end{equation}
A $k$-linear alternating map $\varphi$ is linear in each of its coordinates and satisfies $\varphi(\ldots,v,\ldots,u,\ldots) = -\varphi(\ldots,u,\ldots,v,\ldots)$. The space $\Lambda^k(\Real^n)$ is a linear vector space with a basis of alternating $k$-forms denoted $dx_{i_1}\wedge \cdots \wedge dx_{i_k}$. The way these $k$-forms act on $k$-vectors, $v_1,\ldots,v_k \in \Real^n$ is as a signed volume function:
\begin{equation}\label{e:wedge}
    dx_{i_1}\wedge \cdots \wedge dx_{i_k}(v_1,\ldots,v_k) = \det \brac{dx_{i_r}(v_s)}_{r,s\in [k]}
\end{equation}
Expanding an arbitrary element $\omega\in \Lambda^k(\Real^n)$ in this basis gives 
\begin{equation}
    \omega = \sum_{i_1<i_2<\cdots < i_k} a_{i_1\cdots i_k} \, dx_{i_1}\wedge\cdots \wedge dx_{i_k} = \sum_{I} a_I \, dx_{I}
\end{equation}
where $i_1,\ldots,i_k \in [n]$, $I=(i_1,\ldots,i_k)$ are multi-indices, $a_I$ are scalars, and $dx_I=dx_{i_1}\wedge\cdots \wedge dx_{i_k}$.

The space of \emph{differential} $k$-forms (also called $k$-forms in short), denoted $\gA^k(\Real^n)$, is defined by smoothly assigning to each $x\in \Real^n$ a $k$-linear alternating form $w\in \Lambda^k(\Real^n)$. That is 
\begin{equation}
    w(x) = \sum_{I} f_I(x) dx_I
\end{equation}
where $f_I:\Real^n\too \Real$ are smooth scalar functions. Note that  $\gA^0(\Real^n)$ is the space of smooth scalar functions over $\Real^n$. 
The differential operator can be seen as a linear operator $d:\gA^0(\Real^n)\too \gA^1(\Real^n)$ defined by
\begin{equation}
df(x) = \sum_{i=1}^n \frac{\partial f}{\partial x_i}(x) dx_i
\end{equation}
The exterior derivative $d:\gA^k(\Real^n)\too\gA^{k+1}(\Real^n)$ is a linear differential operator generalizing the differential to arbitrary differential $k$-forms: 
\begin{equation}\label{e:d}
    d\omega(x) = \sum_I df_I \wedge dx_I
\end{equation}
where the exterior product $\omega \wedge \eta$ of two forms $\omega=\sum_I f_I dx_I$, $\eta=\sum_J g_J dx_J$ is defined by extending \eqref{e:wedge} linearly, that is, $\omega\wedge \eta = \sum_{I,J} f_I g_J dx_I\wedge dx_J$. An imporant property of the exterior derivative is that $dd\omega=0$ for all $\omega$. This property can be checked using the definition in \eqref{e:d} and the symmetry of mixed partials, $\tfrac{\partial^2 f}{\partial x_i \partial x_j}=\tfrac{\partial^2 f}{\partial x_j \partial x_i}$.

The hodge operator $\star:\gA^{k}(\Real^n)\too \gA^{n-k}(\Real^n)$ matches to each $k$-form an $n-k$-form by extending the rule
\begin{equation}\label{eq:hodge_star}
    \star (dx_I) = (-1)^\sigma dx_J 
\end{equation}
linearly, where $\sigma=\mathrm{sign}([I,J])$ is the sign of the permutation $(I,J)=(i_1,\ldots,i_k,j_1,\ldots,i_{n-k})$ of $[n]$. The hodge star is (up to a sign) its own inverse, $\star \star \omega= (-1)^{k(n-k)}\omega$, for all $\omega\in \gA^k(\Real^n)$. The $d$ operator has an adjoint operator $\delta:\gA^{k}(\Real^n)\too \gA^{k - 1}(\Real^n)$ defined by 
\begin{equation}
    \delta = (-1)^{n(k+1)+1}\star d \star.
\end{equation}
A vector field $v(x)=(v_1(x),\ldots,v_n(x))$ in $\Real^n$ can be identified with a 1-form $v=\sum_{i=1}^n v_i(x) dx_i$. The divergence of $v$, denoted $\div(v)$ can be expressed using exterior derivative :
\begin{align} \label{e:d_star_div}
    d\star v &= d\sum_{i=1}^n v_i \star dx_i = \sum_{i=1}^n \frac{\partial v_i}{\partial x_i} dx_1\wedge\cdots\wedge dx_n =\mathrm{div}(v) dx_1\wedge\cdots\wedge dx_n,
\end{align}
where in the second equality we used the definition of $\star$, the definition of $d$ and the fact that $dx_i \wedge dx_i=0$ (can be seen from the repeating rows in the matrix inside the determinant in \eqref{e:wedge}). Note that the $n$-form $\mathrm{div}(v) dx_1\wedge \cdots\wedge dx_n$ is identified via $\star$ with the function (\ie, $0$-form) $\mathrm{div}(v)$.

\subsection{Derivations of \eqref{eq:mu} and \eqref{eq:delta_nu}}
\label{app:derivations}

We will need the following identity below, for $\mu \in \Omega^{k}(\gM)$, expressed as 
\[ \mu = \sum_{I} \alpha_I \star (dx^I)\]
\vspace{-12pt}
we have
\begin{equation}
\star \mu = (-1)^{k(n\text{-}k)} \sum_{I} \alpha_I dx^I \label{eq:star_id}
\end{equation}
which follows by linearity of $\star$ over $C^{\infty}(\gM)$ and the identity $\star \star = (-1)^{k(n-k)} \text{id}$.

Here, we derive the expression for \eqref{eq:mu}.
For $\mu \in \Omega^{n-2}(\gM)$, the exterior derivative is given by \eqref{e:d}. However, in this case, since $dx^i \wedge dx^i = 0$, we can expand 
\begin{align}
    d\mu &= \sum_{i,j} d\mu_{i,j} \wedge \star(dx^i\wedge dx^j) \\
    &= \sum_{i,j} \sum_{k} \frac{\partial \mu_k}{\partial x^k} dx^k \wedge \star(dx^i\wedge dx^j)\\
    &= \sum_{i,j} \left[ \frac{\partial \mu_{ij}}{\partial x^{i}} dx^{i} \wedge \star(dx^i\wedge dx^j)  + \frac{\partial \mu_{ij}}{\partial x^{j}} dx^{j} \wedge \star(dx^i\wedge dx^j)  \right]\\
    &= \sum_{i,j} \left[ \frac{\partial \mu_{ij}}{\partial x^{i}} dx^{i} \wedge \star(dx^i\wedge dx^j) \right] + \sum_{i,j} \left[\frac{\partial \mu_{ji}}{\partial x^{i}} dx^{i} \wedge \star(dx^j\wedge dx^i)  \right]\label{eq:dmu_reidx}\\
    &= \sum_{i,j} \left[ \frac{\partial \mu_{ij}}{\partial x^{i}} dx^{i} \wedge \star(dx^i\wedge dx^j) \right] + \sum_{i,j} \left[ \frac{\partial \mu_{ij}}{\partial x^{i}} dx^{i} \wedge \star(dx^i\wedge dx^j)  \right] \label{eq:dmu_anti}\\
    &= 2 \sum_{i,j} \left[ \frac{\partial \mu_{ij}}{\partial x^{i}} dx^{i} \wedge \star(dx^i\wedge dx^j) \right] \\
    &= 2  \sum_i \left[\sum_j \frac{\partial \mu_{ij}}{\partial x^{j}} \star(dx^i)\right]
\end{align}
where \eqref{eq:dmu_reidx} follows by re-indexing $i$ as $j$, and then \eqref{eq:dmu_anti} by anti-symmetry $\mu_{ij} = - \mu_{ji}$ and $\star(dx^i \wedge dx^j) = - \star (dx^j \wedge dx^i)$ (the sign flips cancel). Applying $\star$ then gives (by \eqref{eq:star_id}), 
\[ \star d\mu = \star \left[ 2  \sum_i \sum_j \frac{\partial \mu_{ij}}{\partial x^{j}} \star(dx^i)\right] = 2 (-1)^{n-1} \sum_i \sum_j \frac{\partial \mu_{ij}}{\partial x^{j}} dx^i \]
 To derive equation \eqref{eq:delta_nu}, it suffices to start by noting $\delta = \star d \star$, and so
\begin{align}
\delta \nu &= \star d \star \nu \\
&= \star d \star \sum_i \nu_i \star(dx^i)\\
&= \star d (-1)^{n-1} \sum_i \nu_i dx^i\\
&= (-1)^{n-1} \star \sum_i \sum_j {\partial \nu_i \over \partial x^j} dx^j \wedge dx^i\\
&= (-1)^{n-1} \sum_i \sum_j {\partial \nu_i \over \partial x^j} \star(dx^j \wedge dx^i)\\
&= - (-1)^{n-1} \sum_i \sum_j {\partial \nu_i \over \partial x^j} \star(dx^i \wedge dx^j)\\
&= -(-1)^{n-1} \left( \sum_{i < j} {\partial \nu_i \over \partial x^j} \star(dx^i \wedge dx^j) + \sum_{i > j} {\partial \nu_i \over \partial x^j} \star(dx^i \wedge dx^j) \right)\\
&= -(-1)^{n-1} \left( \sum_{i < j} {\partial \nu_i \over \partial x^j} \star(dx^i \wedge dx^j) -  \sum_{i < j} {\partial \nu_j \over \partial x^i} \star(dx^i \wedge dx^j) \right)\\
&= - (-1)^{n-1} \sum_{i < j} \left[ {\partial \nu_i \over \partial x^j}  - {\partial \nu_j \over \partial x^i} \right] \star(dx^i \wedge dx^j)
\end{align}
then since $\star(dx^i \wedge dx^j) = - \star(dx^j \wedge dx^i)$, we have 
\begin{align}
    &\text{ } - (-1)^{n-1} \sum_{i < j} \left[ {\partial \nu_i \over \partial x^j}  - {\partial \nu_j \over \partial x^i} \right] \star(dx^i \wedge dx^j)\\
    &= - (-1)^{n-1} \left( \frac{1}{2} \sum_{i < j} \left[ {\partial \nu_i \over \partial x^j}  - {\partial \nu_j \over \partial x^i} \right] \star(dx^i \wedge dx^j) + \frac{1}{2} \sum_{i < j} \left[ {\partial \nu_j \over \partial x^i}  - {\partial \nu_i \over \partial x^j} \right] \star(dx^j \wedge dx^i) \right)\\
    &= - (-1)^{n-1} \left( \frac{1}{2} \sum_{i < j} \left[ {\partial \nu_i \over \partial x^j}  - {\partial \nu_j \over \partial x^i} \right] \star(dx^i \wedge dx^j) + \frac{1}{2} \sum_{i > j} \left[ {\partial \nu_i \over \partial x^j}  - {\partial \nu_j \over \partial x^i} \right] \star(dx^i \wedge dx^j) \right)\\
    &=  - (-1)^{n-1} \left( \frac{1}{2} \sum_{i, j} \left[ {\partial \nu_i \over \partial x^j}  - {\partial \nu_j \over \partial x^i} \right] \star(dx^i \wedge dx^j) \right)\\
\end{align}
\section{Proofs and derivations}
\label{apx:derivations}

\subsection{Universality}
\label{app:universality}

\universalitythm*
\begin{proof}
First, consider a divergence-free vector field and its representation as a 1-form $ v\in\gA^1(\sT)$. Then $v$ being divergence-free means $d\star v=0$. We denote by $c=(c_1,\ldots,c_n)\in\Real^n$ a constant vector field; note that $c$ is also a well defined vector field over $\sT$. We will use the notation $c$ to denote the corresponding constant 1-form. We claim $\star v = d\mu + \star c$, where $\mu\in\gA^{n-2}(\sT)$, and $\star c\in\gA^{n-1}(\gS)$ is constant. This can be shown with Hodge decomposition \citep{morita2001geometry} of $\star v\in \gA^{n-1}(\sT)$:
\begin{equation}
\star v = d\mu + \delta \tau + h,
\end{equation}
where $\mu\in \gA^{n-2}(\sT)$, $\tau\in \gA^{n}(\sT)$ and $h\in\gA^{n-1}$ is a harmonic $n$-1-form. Note that the harmonic $n$-1-forms over $\gT$ are simply constant. Taking $d$ of both sides leads to 
\begin{equation}
d\star v = d\delta \tau.
\end{equation}
Since we assumed $v$ is div-free, $d\star v=0$ and we get $d\delta\tau=0$. Since $d\delta \tau=0=\delta \delta \tau$ then $\delta \tau$ is harmonic (constant) as well. 

So far we showed $\star v = d\mu + \star c$, which shows the universality of the matrix-field representation (up to a constant). To show universality of the vector-field representation we need to show that: 
\begin{equation}
\set{d\mu+\star c \ \big \vert \  \mu\in \gA^{n-2}(\sT)}=\set{d\delta\nu+\star c \ \big  \vert \ \nu\in \gA^{n-1}(\sT)}
\end{equation}
The left inclusion $\supset$ is true since $\delta \nu \in \gA^{n-2}(\gT)$. For the right inclusion $\subset$ we take an arbitrary $\mu \in \gA^{n-2}(\sT)$ and decompose it with Hodge:
$\mu = d\omega + \delta \nu + h$, where $\omega\in \gA^{n-3}(\sT)$, $\nu\in\gA^{n-1}(\sT)$, and $h\in \gA^{n-2}(\sT)$ is harmonic. Taking $d$ of both sides leaves us with $d\mu = d\delta \nu$ that shows that $d\mu+\star c$ is included in the right set.

\end{proof}

\subsection{Stabilizing training for fluid simulations}\label{app:stable_losses}

In order to stabilize training, we can modify the loss terms
$L_F,L_G,L_I$ to avoid division by $\rho$. As before $v = [\rho,\rho u]$,and
\begin{align}
    \tilde L_{F} &= \norm{\rho^2(\rho u)_t - \rho (\rho_t) \rho u + \rho [D(\rho u) (\rho u]) - [\nabla \rho \otimes \rho u](\rho u) +  \rho^2\nabla p}_{\Omega}^2 \\
     \tilde L_{\div} &= \norm{\tilde \nabla \rho \cdot v}_{\Omega}\\
    \tilde L_{I} &= \norm{\rho u(0,\cdot) - \rho_0 u_0(0,\cdot)}_{\Omega}^2 + \norm{\rho(0,\cdot) - \rho_0(0,\cdot) }_{\Omega}^2 \\
    \tilde L_{G} &= \norm{\rho u\cdot n}_{\partial \Omega}^2 
\end{align}
In practice, we noticed this improved training stability significantly, which is intuitive since the possibility of a division by 0 is removed. The derivation of $\tilde L_{G}$, and $\tilde L_I$ is simply scaling by $\rho$. We derive $\tilde L_F$ and $\tilde L_\div$ by repeatedly applying the product rules for the Jacobian and divergence operators and solving for $\rho^{2,3}$ scaled copies of the residuals. Below, we use the convention that the gradient and divergence operators only act in spatial variables, but the $\tilde \nabla$, $\tilde \div$ operators include time
\begin{equation}
\tilde \nabla \rho(t,x) = \begin{pmatrix}\frac{\partial \rho}{\partial t} (t,x) \\ \frac{\partial\rho}{\partial x_1}(t,x)  \\ \vdots \end{pmatrix} \qquad \tilde \div(u) = \frac{\partial u_0}{\partial t} + \sum_{i=1}^d \frac{\partial u_i}{\partial x_i}
\end{equation}

\textbf{Derivation of $\tilde L_\div$}:\\
To derive $\tilde L_\div$, we consider 
\begin{equation} 
\tilde \div\left( \rho \begin{pmatrix} 1 \\ u \end{pmatrix} \right) = \tilde \nabla \rho \cdot \begin{pmatrix} 1 \\ u \end{pmatrix}  + \rho \tilde \div \begin{pmatrix} 1 \\ u \end{pmatrix} 
\end{equation}
by construction $\div \left( \rho \begin{pmatrix} 1 \\ u \end{pmatrix} \right) = 0$, and since $\tilde \div \begin{pmatrix} 1 \\ u \end{pmatrix} = \div(u)$, multiplying both sides by $\rho$ we find 
\begin{equation} 
0 = \tilde \nabla \rho \cdot v + \rho^2 \div(u) \implies \rho^2 \div(u) = -\tilde \nabla \rho \cdot v 
\end{equation}

\textbf{Derivation of $\tilde L_F$}
To derive $\tilde L_F$, we will start at the end. We multiply the momentum term of the Euler system
\begin{equation} 
\frac{\partial u}{\partial t} + [Du]u + \frac{\nabla p}{\rho} = 0 \implies \rho^3 {\partial u \over \partial t} + \rho^3 [Du]u + \rho^2 \nabla p = 0
\end{equation}
we start by applying the product rule to $\rho u$
\begin{equation}
    {\partial (\rho u) \over \partial t} = \frac{\partial \rho}{\partial t} u + \rho \frac{\partial u }{\partial t}
\end{equation}
multiplying by $\rho^2$ and solving for $\rho^3 \frac{\partial u}{\partial t}$ yields
\begin{equation} \label{eq:apx-deriv1}
\rho^3 \frac{\partial u }{\partial t} =  \rho^2 {\partial (\rho u) \over \partial t} - \rho \frac{\partial \rho}{\partial t} \rho u 
\end{equation}
which can be computed without dividing by $\rho$. Now we apply the Jacobian scalar product rule
\begin{equation} 
D(\rho u) = \nabla \rho \otimes u + \rho Du
\end{equation}
contracting with $\rho^2 u$ yields that 
\begin{equation} 
D(\rho u) \rho^2 u = [\nabla p \otimes (\rho u)](\rho u) + \rho^3 [Du]u
\end{equation}
which gives 
\begin{equation} \label{eq:apx-deriv2} 
\rho^3 [Du]u = D(\rho u) \rho^2 u - [\nabla p \otimes (\rho u)](\rho u)
\end{equation}
which can also be computed without invoking division by $\rho$. Together, \eqref{eq:apx-deriv1} and \eqref{eq:apx-deriv2} yield $\tilde L_F$.

\section{Implementation Details}
\label{apx:details}
\subsection{Tori Example Details}
For the Tori example, we used the matrix formulation of the NCL model. The matrix was parameterized as the output of a 8 layer, 512-wide Multi-Layer Perceptron with softplus activation. We trained this model for 600,000 steps of stochastic gradient descent using a batch size of 1000. The weight vector $\gamma$ was fixed with $\gamma_F = 3\times 10^{-3}$, $\gamma_I = 30$, $\gamma_{\div} = 0.01$. Instead of choosing a fixed set of colocation points (as is common in the PINN literature see \cite{raissi_physics_2017}), we sampled uniformly on the unit square $[0,1]^2$.  For the Finite-Element reference solution, we solved the system on a 50 x 50 grid with periodic boundary conditions implemented with a mixed Lagrange element scheme. The splitting scheme used was the inviscid case of \cite{guermond_projection_2000}, with time step $dt=0.001$. 


\begin{figure}
     \centering
     \includegraphics[width=\linewidth]{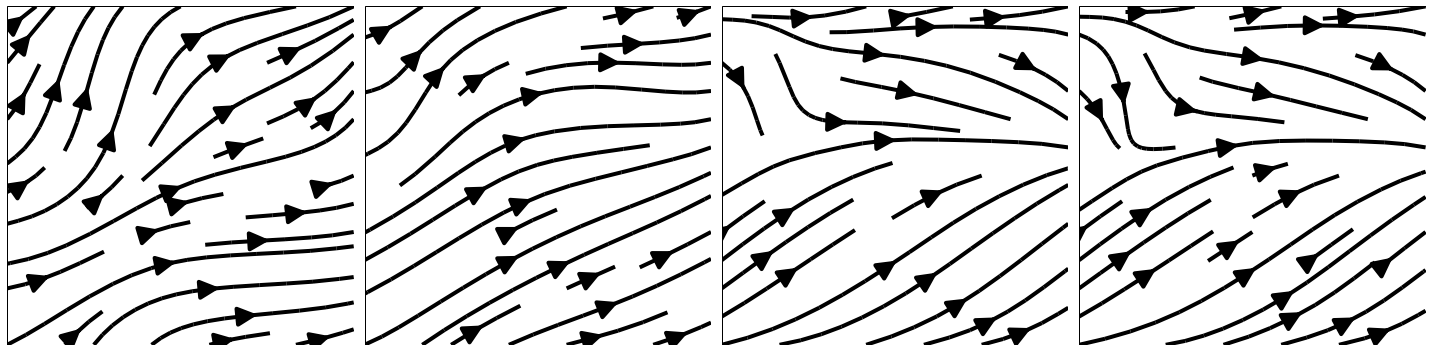}\\
     \includegraphics[width=\linewidth]{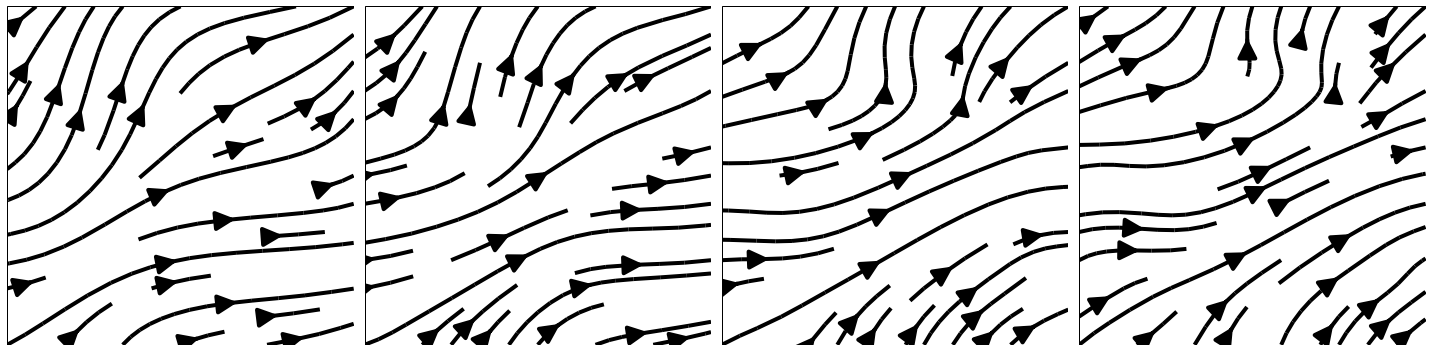}\\
     \includegraphics[width=\linewidth]{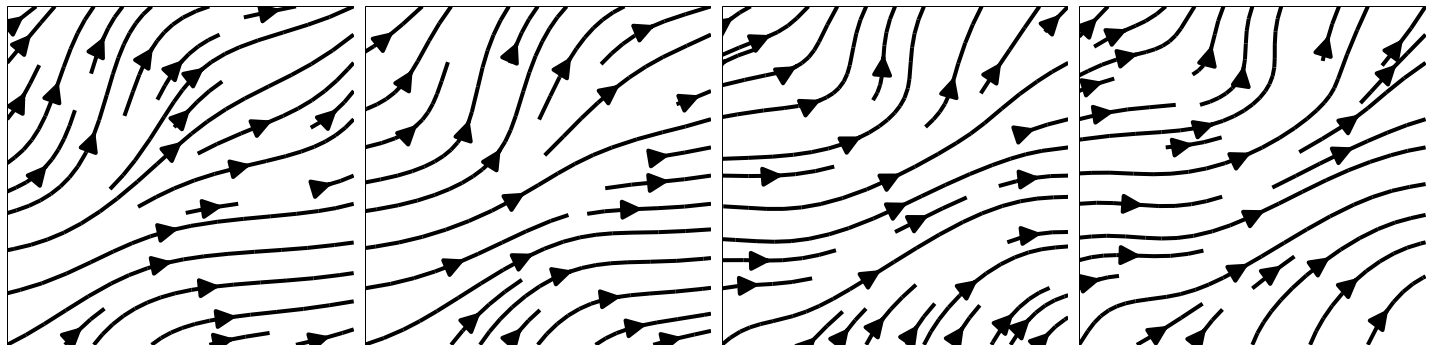}\\
    \caption{Streamplots showing the velocity field from a Physics Informed Neural Network (\textit{top}), our method (\textit{middle}) vs a reference FEM solution (\textit{bottom}). While both models minimize the loss effectively and fit the initial conditions, the PINN fails to learn the correct evolution of the velocity.}
    \label{fig:fluids:2d-tori-experiment-velocity}
\end{figure}


\subsection{3d Ball Example Details}
For the comparison in \Cref{fig:apx:large-3d} we used a 4-layer, 128 wide feed forward network for both the Curl PINN and the NCL. Stopped training both models after 10000 steps of stochastic gradient descent with batch size of 1000. While this is much less than \Cref{sec:fluids:manifold-experiments}, the difference can be explained by the initial condition being much less complex. For the NCL model, we used $\gamma_F = 0.1, \gamma_{\div}=0.1, \gamma_G = 0.1, \gamma_I = 30$. For the CURL model, we used  $\gamma_F = 0.1,\gamma_G = 0.1, \gamma_\text{cont}= 10, \gamma_I = 30$

A larger plot of the comparison is shown in \Cref{fig:apx:large-3d}.

\begin{figure}
    \centering
    \begin{subfigure}[b]{0.7\linewidth}
    \begin{subfigure}[b]{\linewidth}
    \includegraphics[width=\linewidth]{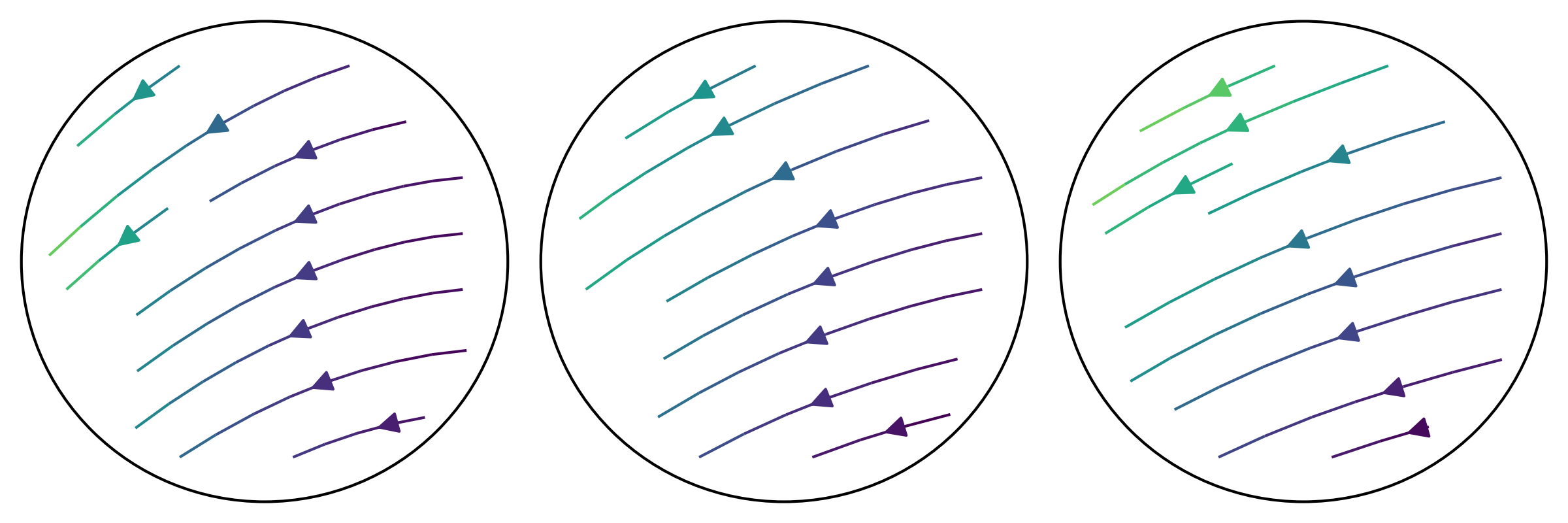}
    \end{subfigure}
    \begin{subfigure}[b]{\linewidth}
    \includegraphics[width=\linewidth]{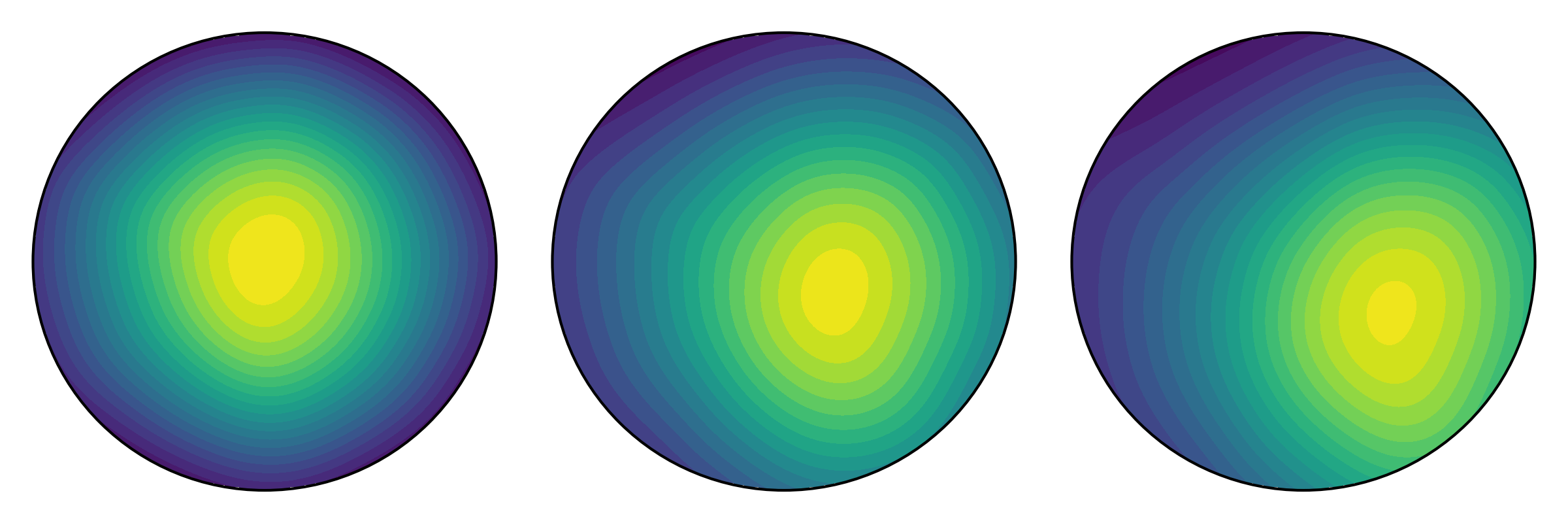}
    \end{subfigure}
    \caption{Curl Model: Flow (top) and Density (bottom) at $t=0,0.25,0.5$}
    \end{subfigure}
    
    \begin{subfigure}[b]{0.7\linewidth}
    \begin{subfigure}[b]{\linewidth}
    \includegraphics[width=\linewidth]{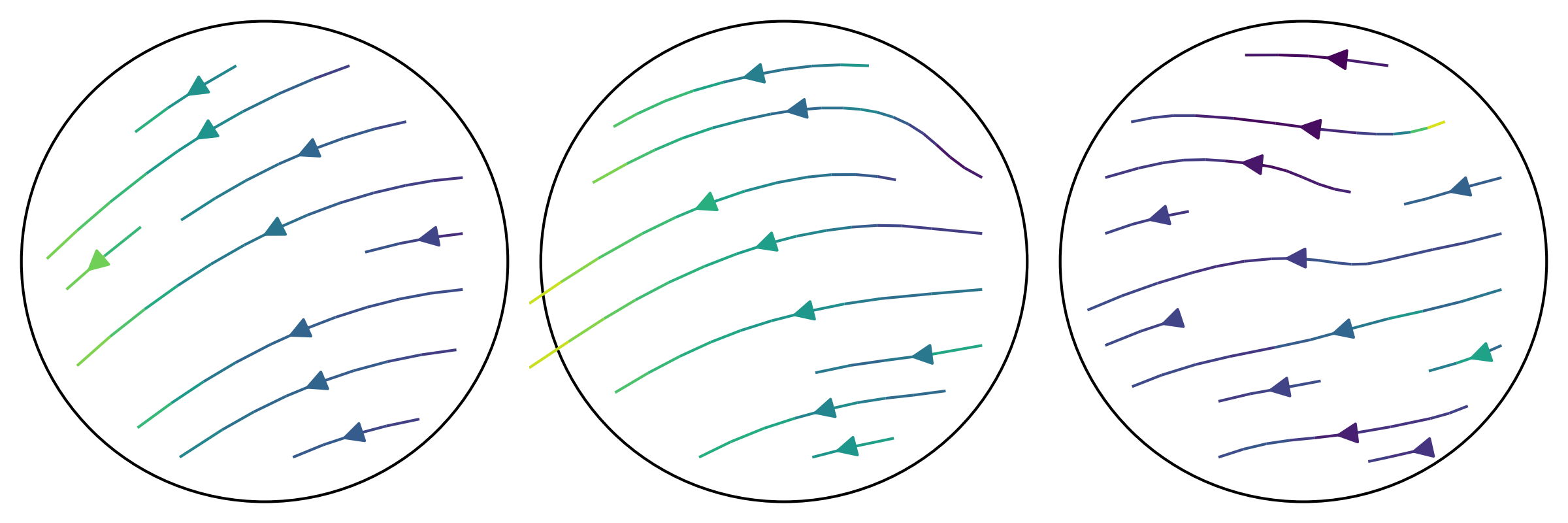}
    \end{subfigure}
    \begin{subfigure}[b]{\linewidth}
    \includegraphics[width=\linewidth]{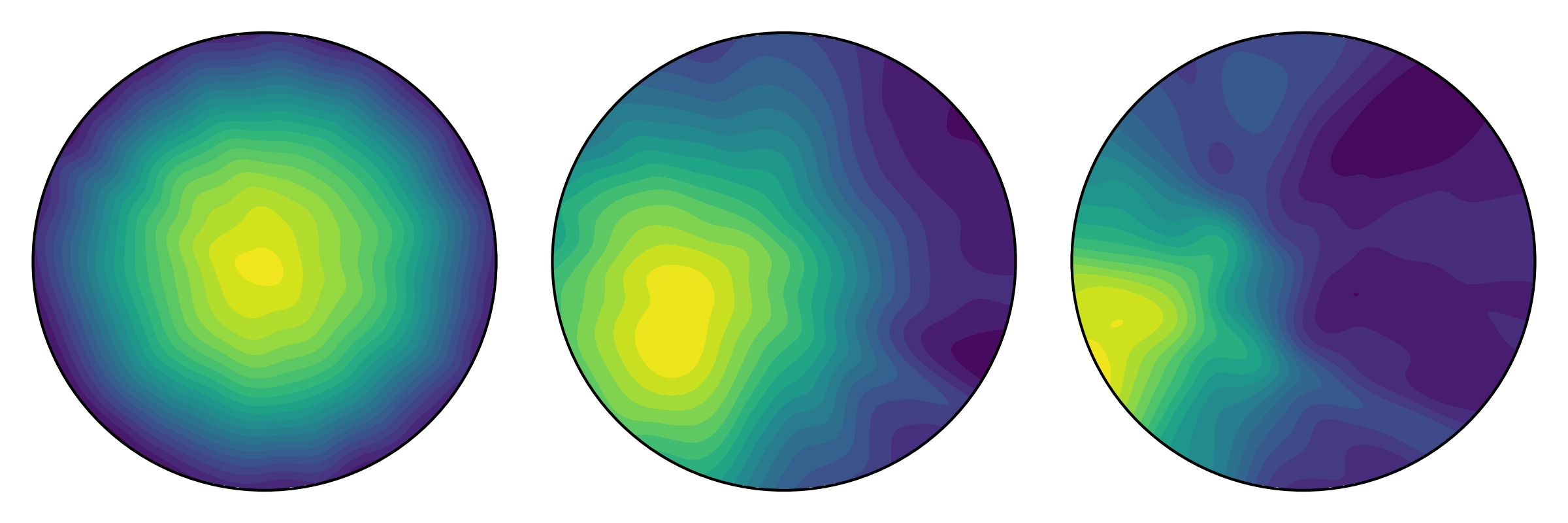}
    \end{subfigure}
    \caption{Our model (NCL): Flow (top) and Density (bottom) at $t=0,0.25,0.5$}
    \end{subfigure}

    \caption{Larger version of comparison shown in \Cref{sec:fluids:3d-experiments}. While our model fits the initial conditions and convects the density along the flow lines as expected, the curl model fails to do so.}
    \label{fig:apx:large-3d}
\end{figure}

\begin{figure}
    \centering
    \includegraphics[width=0.7\linewidth]{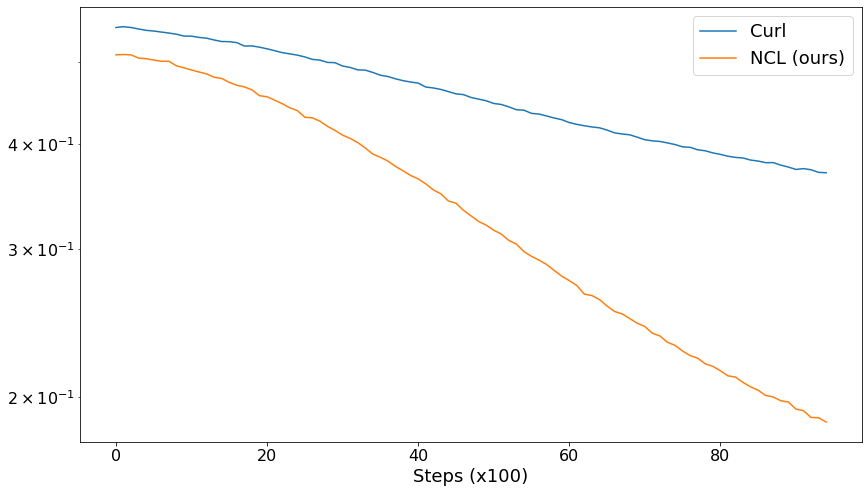}
    \caption{Training loss for NCL model (ours) plotted against the Curl model, for the 3D unit ball fluid experiment. Both models achieve a similar order of magnitude loss but exhibit qualitatively different results.  }
    \label{fig:apx:3d-train-loss}
\end{figure}

\end{document}

%% file: math_commands.tex
\usepackage{amsmath,amsfonts,bm}

\newcommand{\set}[1]{\left\{#1\right\}}

\newcommand{\brac}[1]{\left [#1\right ]}

\newcommand{\Real}{\mathbb R}

\newcommand{\too}{\rightarrow}

\def\eqref#1{equation~\ref{#1}}

\def\1{\bm{1}}

\DeclareMathAlphabet{\mathsfit}{\encodingdefault}{\sfdefault}{m}{sl}
\SetMathAlphabet{\mathsfit}{bold}{\encodingdefault}{\sfdefault}{bx}{n}

\def\gA{{\mathcal{A}}}

\def\gM{{\mathcal{M}}}

\def\gS{{\mathcal{S}}}
\def\gT{{\mathcal{T}}}

\def\sT{{\mathbb{T}}}

\newcommand{\E}{\mathbb{E}}

\newcommand{\R}{\mathbb{R}}

